%% file: journal_main_arxiv.tex
\documentclass[twoside,11pt]{article}

\usepackage{times}
\usepackage{fullpage}

\usepackage{graphicx} 
\usepackage{subfigure}

\usepackage[round]{natbib}

\usepackage{algorithm}
\usepackage{algorithmic}
\usepackage{url}

\usepackage{multirow}
\usepackage{enumitem}
\usepackage{xspace}
\usepackage{color}
\usepackage{amssymb,amsfonts,latexsym,amsmath,amsthm}
\usepackage{xspace}

\usepackage{bbm}

\input{math_definition}

\let\originalparagraph\paragraph
\renewcommand{\paragraph}[2][.]{\originalparagraph{#2#1}}

\def\arxiv{1}

\begin{document}

\title{Escaping the Curse of Dimensionality in Similarity Learning: Efficient Frank-Wolfe Algorithm and Generalization Bounds}

\author{
Kuan Liu\thanks{Google Inc., USA, \texttt{liukuan@google.com}. Most of this
work was done when the author was affiliated with the Information Sciences
Institute, University of Southern California, USA}~~and Aur\'elien Bellet\thanks{INRIA, \texttt{aurelien.bellet@inria.fr}.}
}

\date{}

\maketitle

\begin{abstract}
\input{abstract}
\end{abstract}

\input{intro}

\input{related}

\input{approach}

\input{generalization}

\input{exp}
\input{conclu}

\input{acks}

\bibliographystyle{plainnat}
\bibliography{main}

\appendix
\renewcommand\thesection{Appendix~\Alph{section}}
\input{appendix}

\end{document}

%% file: math_definition.tex
\usepackage{amsmath,amsfonts}
\usepackage{amsthm,amssymb,amsopn}
\usepackage{bm} 

\newcommand{\vct}[1]{\boldsymbol{#1}} 
\newcommand{\mat}[1]{\boldsymbol{#1}} 

\newcommand{\T}{^{\textrm T}} 

\DeclareMathOperator*{\argmax}{arg\,max}
\DeclareMathOperator*{\argmin}{arg\,min}

\newcommand{\eat}[1]{}

\theoremstyle{plain}

\newtheorem{lemma}{Lemma}
\newtheorem{theorem}{Theorem}
\newtheorem{prop}{Proposition}
\newtheorem{cor}{Corollary}

\newtheorem{rem}{Remark}

\newcommand{\innerp}[1]{\langle #1 \rangle}

\DeclareMathOperator{\diam}{diam}
\DeclareMathOperator{\diag}{diag}
\DeclareMathOperator*{\conv}{conv}

\DeclareMathOperator*{\expec}{\mathbb{E}}

\newcommand{\mldiag}{\textsc{diag}\xspace}
\newcommand{\hdsl}{\textsc{hdsl}\xspace}
\newcommand{\ident}{\textsc{dot}\xspace}
\newcommand{\pcaml}{\textsc{pca+oasis}\xspace}

\newcommand{\rpml}{\textsc{rp+oasis}\xspace}

\newcommand{\svm}{\textsc{svm}\xspace}

%% file: abstract.tex

Similarity and metric learning provides a principled approach to construct a task-specific similarity from weakly supervised data. However, these methods are subject to the curse of dimensionality: as the number of features grows large, poor generalization is to be expected and training becomes intractable due to high computational and memory costs.
In this paper, we propose a similarity learning method that can efficiently deal with high-dimensional sparse data. This is achieved through a parameterization of similarity functions by convex combinations of sparse rank-one matrices, together with the use of a greedy approximate Frank-Wolfe algorithm which provides an efficient way to control the number of active features. We show that the convergence rate of the algorithm, as well as its time and memory complexity, are independent of the data dimension. We further provide a theoretical justification of our modeling choices through an analysis of the generalization error, which depends logarithmically on the sparsity of the solution rather than on the number of features. Our experiments on datasets with up to one million features demonstrate the ability of our approach to generalize well despite the high dimensionality as well as its superiority compared to several competing methods. 


%% file: intro.tex

\section{Introduction}



High-dimensional and sparse data are commonly encountered in many applications of machine learning, such as computer vision, bioinformatics, text mining and behavioral targeting. To classify, cluster or rank data points, it is important to be able to compute semantically meaningful similarities between them. However, defining an appropriate similarity measure for a given task is often difficult as only a small and unknown subset of all features are actually relevant. For instance, in drug discovery studies, chemical compounds are typically represented by a large number of sparse features describing their 2D and 3D properties, and only a few of them play in role in determining whether the compound will bind to a particular target receptor \citep{Leach2007a}. In text classification and clustering, a document is often represented as a sparse bag of words, and only a small subset of the dictionary is generally useful to discriminate between documents about different topics. Another example is targeted advertising, where ads are selected based on fine-grained user history \citep{Chen2009a}.


Similarity and metric learning \citep{Bellet2015c} offers principled approaches to construct a task-specific similarity measure by learning it from weakly supervised data, and has been used in many application domains. The main theme in these methods is to learn the parameters of a similarity (or distance) function such that it agrees with task-specific similarity judgments (e.g., of the form ``data point $\vct{x}$ should be more similar to $\vct{y}$ than to $\vct{z}$''). 
To account for correlations between features, similarity and metric learning typically estimates a number of parameters which is quadratic in the data dimension $d$.  When data are high-dimensional, these methods are thus particularly affected by the so-called ``curse of dimensionality'', which manifests itself at both the algorithmic and generalization levels. On the one hand, training the similarity quickly becomes infeasible due to a quadratic or cubic complexity in $d$. In fact, the $O(d^2)$ parameters may not even fit in memory. On the other hand, putting aside the training phase, learning so many parameters would lead to severe overfitting and poor generalization performance (especially for sparse data where some features are rarely observed). Simple workarounds have been used to address this limitation, such as projecting the data into a low-dimensional space before learning the similarity \cite[see e.g.][]{Davis2007,Weinberger2009,Guillaumin2009}. However, such heuristics do not provide satisfactory solutions: they often hurt the performance and make the resulting similarity function difficult to interpret.

In this paper, we propose a novel method to learn a bilinear similarity function $S_{\mat{M}}(\vct{x},\vct{x}') = \vct{x}\T\mat{M}\vct{x}'$ directly in the original high-dimensional space while escaping the curse of dimensionality. This is achieved by combining three ingredients: the sparsity of the data, the parameterization of $\mat{M}$ as a convex combination of rank-one matrices with a special sparsity structure, and an approximate Frank-Wolfe procedure \citep{Frank1956,Jaggi2013} to learn the similarity parameters. 
The resulting algorithm greedily incorporates one pair of features at a time into the learned similarity, providing an efficient way to filter out irrelevant features as well as to guard against overfitting through early stopping. Remarkably, the convergence rate of the algorithm as well as its time and memory complexity are all independent of the dimension $d$. The resulting similarity functions are extremely sparse, which makes them fast to compute and easier to interpret.

We provide strong theoretical and empirical evidence of the usefulness of our approach. On the theory part, we perform a generalization analysis of the solution returned by our algorithm after a given number of iterations. We derive excess risk bounds with respect to the minimizer of the expected risk which confirm that our modeling choices as well as our Frank-Wolfe algorithm and early stopping policy provide effective ways to avoid overfitting in high dimensions. A distinctive feature of the generalization bound we obtain is the adaptivity of its model class complexity term to the actual sparsity of the approximate solution found by our algorithm, again removing the dependence on the dimension $d$.
We also evaluate the proposed approach on several synthetic and real datasets with up to one million features, some of which have a large proportion of irrelevant features. To the best of our knowledge, it is the first time that a full similarity or distance metric is learned directly on such high-dimensional datasets without first reducing dimensionality.
Our experiments show that our approach is able to generalize well despite the high dimensionality, and even to recover the ground truth similarity function when the training similarity judgments are sufficiently informative. Furthermore, our approach clearly outperforms both a diagonal similarity learned in the original space and a full similarity learned in a reduced space (after PCA or random projections). Finally, we show that our similarity functions can be extremely sparse (in the order of {0.0001\%} of nonzero entries), thereby drastically reducing the dimension while also providing an opportunity to analyze the importance of the original features and their pairwise interactions for the problem at hand.

The present work extends a previously published conference paper \citep{Liu2015a} by providing additional technical and experimental results. Firstly, we present a novel generalization analysis which further backs up our approach from a statistical learning point of view. Secondly, we conduct experiments on high-dimensional synthetic data showing that our approach generalizes well as the dimensionality increases and can even accurately recover the ground truth notion of similarity. Finally, we extend the discussion of the related work and provide additional details on algorithms and proofs.

The paper is organized as follows. Section \ref{sRelate} introduces some background and related work on similarity learning and Frank-Wolfe algorithms. Section \ref{sApproach} describes our problem formulation, the proposed algorithm and its analysis. Generalization bounds are established in Section \ref{sGen}. Finally, Section \ref{sExp} describes our experimental results, and we conclude in Section \ref{sConclude}.

%% file: related.tex

\section{Background and Related Work}
\label{sRelate}

In this section, we review some background and related work in metric and similarity learning (Section~\ref{sec:related_metric}) and the Frank-Wolfe algorithm (Section~\ref{sec:related_fw}).

\subsection{Metric and Similarity Learning}
\label{sec:related_metric}


Metric and similarity learning has attracted a lot of interest over the past ten years. The great majority of work has focused on learning either a Mahalanobis distance $d_{\mat{M}}(\vct{x},\vct{x}') = (\vct{x}-\vct{x}')\T\mat{M}(\vct{x}-\vct{x}')$ where $\mat{M}$ is a symmetric positive semi-definite (PSD) matrix, or a bilinear similarity $S_{\mat{M}}(\vct{x},\vct{x}') = \vct{x}\T\mat{M}\vct{x}'$ where $\mat{M}$ is often taken to be an arbitrary $d\times d$ matrix.
A comprehensive survey of existing approaches can be found in \citep{Bellet2013}. We focus below on the two topics most relevant to our work: (i) efficient algorithms for the high-dimensional setting, and (ii) the derivation of generalization guarantees for metric and similarity learning. 

\paragraph{Metric learning in high dimensions}

Both Mahalanobis distance metric learning and bilinear similarity learning require estimating $O(d^2)$ parameters, which is undesirable in the high-dimensional setting for the reasons mentioned earlier. In practice, it is thus customary to resort to dimensionality reduction (such as PCA, SVD or random projections) to preprocess the data when it has more than a few hundred dimensions \citep[see e.g.,][]{Davis2007,Weinberger2009,Guillaumin2009,Ying2012,Wang2012b,Lim2013,Qian2014,Liu2015b,Yao2018a}. Although this strategy can be justified formally in some cases \citep{Liu2015b,Qian2015a}, the projection may intertwine useful features and irrelevant/noisy ones and thus hurt the performance of the resulting similarity function. It also makes it hard to interpret and use for data exploration, preventing the discovery of knowledge that can be valuable to domain experts.

There have been very few satisfactory solutions to this essential limitation. The most drastic strategy is to learn a diagonal matrix $\mat{M}$  \citep{Schultz2003,Gao2014}, which is very restrictive as it amounts to a simple weighting of the features.
Instead, some approaches assume an explicit low-rank decomposition $\mat{M}=\mat{L}\T\mat{L}$ and learn $\mat{L}\in\mathbb{R}^{r\times d}$ in order to reduce the number of parameters \citep{Goldberger2004,Weinberger2009,Kedem2012}. This results in nonconvex formulations with many local optima \citep{Kulis2012}, and requires to tune $r$ carefully. Moreover, the training complexity still depends on $d$ and can thus remain quite large. Another direction is to learn $\mat{M}$ as a combination of rank-one matrices. In particular, \citet{Shi2014} generate a set of rank-one matrices from the training data and then learn a metric as a sparse combination. However, as the dimension increases, a larger dictionary is needed and can be expensive to generate. Some other work has studied sparse and/or low-rank regularization to reduce overfitting in high dimensions \citep{Rosales2006,Qi2009,Ying2009} but this does not in itself reduce the training complexity of the algorithm. \citet{Zhang2017a} proposed a stochastic gradient descent solver together with low-rank regularization in an attempt to keep the intermediate solutions low-rank. The complexity per iteration of their approach is linear in $d$ but cubic in the rank of the current solution, which quickly becomes intractable unless the regularization is very strong.

Finally, some greedy algorithms for metric learning have been proposed in the literature to guarantee a tighter bound on the rank of intermediate solutions. \citet{Atzmon2015} use a block coordinate descent algorithm to update the metric one feature at a time. \citet{Shen2012} selects rank-one updates in a boosting manner, while DML-eig \citep{Ying2012} and its extension DML-$\rho$ \citep{Cao2012} rely on a greedy Frank-Wolfe algorithm to optimize over the set of PSD matrices with unit trace. However, these greedy methods still suffer from a computational cost of $O(d^2)$ per iteration and are thus unsuitable for the high-dimensional setting we consider in this work. In contrast, we will propose an algorithm which is \emph{linear in the number of nonzero features} and can thus be efficiently applied to high-dimensional sparse data.

\paragraph{Generalization bounds for metric learning} The derivation of generalization guarantees for metric and similarity learning has been investigated in the supervised setting, where the metric or similarity is learned from a labeled dataset of $n$ points by (regularized) empirical risk minimization. For a given family of loss functions, the results generally bound the maximal deviation between the expected risk (where the expectation is taken over the unknown data distribution) and the empirical risk of the learned metric.\footnote{This is in contrast to a different line of work, inspired by the problem of ordinal embedding, which aims to learn a metric which correctly orders a \emph{fixed set of known points} \citep[see for instance][]{Jain2017a}} These bounds are generally of order $O(1/\sqrt{n})$. 

Several technical tools have been used to address the challenge of learning from dependent pairs/triplets, leading to different trade-offs in terms of tightness, generality, and dependence on the feature dimension $d$.
The results of \citet{Jin2009a} apply only under Frobenius norm regularization of $\mat{M}$ and have a $\sqrt{d}$ factor in the rate. Using an adaptation of algorithmic robustness, \citet{Bellet2015a} obtain bounds which hold also for sparsity-inducing regularizers but with a covering number term that can be exponential in the dimension. \citet{Bian2011a} rely on assumptions on the data distribution and do not show an explicit dependence on the dimension.
\citet{Cao2012a} derive bounds based on Rademacher complexity and maximal deviation results for $U$-statistics \citep{Clemencon2008a}. Depending on the regularization used, the dependence on the dimension $d$ ranges from logarithmic to linear.
\cite{Verma2015} show that the $\sqrt{d}$ factor of \citet{Jin2009a} is in fact unavoidable in the worst case without some form of regularization (or restriction of the hypothesis class). They derive bounds which do not depend on the dimension $d$ but on the Frobenius norm of the optimal parameter $\mat{M}$. Note however that their analysis assumes that the metrics are learned from a set of i.i.d. pairs or triplets, which is rarely seen in practice.

In all the above work, generalization in metric learning is studied independently of the algorithm used to solve the empirical risk minimization problem, and none of the bounds are adaptive to the actual sparsity of the solution. In contrast, we will show that one can use early stopping in our algorithm to control the complexity of the hypothesis class so as to make the bounds independent of the dimension $d$, effectively balancing the empirical (optimization) error and the generalization error.

\subsection{Frank-Wolfe Algorithms}
\label{sec:related_fw}

The Frank-Wolfe (FW) algorithm was originally introduced by \citet{Frank1956} and further generalized by \citet{Clarkson2010} and \citet{Jaggi2013}. FW aims at solving constrained optimization problems of the following general form:
\begin{equation}
\label{eq:general_pb}
\min_{\mat{M}\in\mathcal{D}} f(\mat{M}),
\end{equation}
where $f$ is a convex and continuously differentiable function, and the feasible domain $\mathcal{D}$ is a convex and compact subset of some Hilbert space equipped with inner product $\innerp{\cdot,\cdot}$.

\begin{algorithm}[t]
\caption{Standard Frank-Wolfe algorithm}
\label{alg:gen_fw}
\begin{algorithmic}[0]
\STATE \textbf{Input:} Initial point $\mat{M}^{(0)} \in \mathcal{D}$
\FOR{$k = 0,1,2,\dots$}
    \STATE $\mat{S}^{(k)} \leftarrow \argmin_{\mat{S}\in\mathcal{D}}\innerp{\mat{S},\nabla f(\mat{M}^{(k)})}$
    \STATE $\gamma^{(k)} \leftarrow \frac{2}{k+2}$ (or determined by line search)
    \STATE $\mat{M}^{(k+1)} \leftarrow (1-\gamma^{(k)}) \mat{M}^{(k)} + \gamma^{(k)} \mat{S}^{(k)}$
\ENDFOR
\end{algorithmic}
\end{algorithm}

Starting from a feasible initial point $\mat{M}^{(0)}\in\mathcal{D}$, the standard FW algorithm iterates over the following steps. First, it finds the feasible point $\mat{S}^{(k)}\in\mathcal{D}$ which minimizes the linearization of $f$ at the current point $\mat{M}^{(k)}$:
\begin{equation}
    \label{eq:sub}
    \mat{S}^{(k)} \in \argmin_{\mat{S}\in\mathcal{D}} ~ \innerp{\mat{S},\nabla f(\mat{M}^{(k)})}.
\end{equation}
The next iterate $\mat{M}^{(k+1)}$ is then constructed as a convex combination of $\mat{M}^{(k)}$ and $\mat{S}^{(k)}$, where the relative weight of each component is given by a step size $\gamma^{(k)}$. The step size can be decreasing with the iteration number $k$ or set by line search. The overall algorithm is summarized in Algorithm~\ref{alg:gen_fw}. FW is guaranteed to converge to an optimal solution of \eqref{eq:general_pb} at rate $O(1/k)$, see for instance \citep{Jaggi2013} for a generic and concise proof.

Unlike projected gradient, FW is a \emph{projection-free} algorithm: each iterate $\mat{M}^{(k)}$ is feasible by construction since it is a convex combination of elements of $\mathcal{D}$. Instead of computing projections onto the feasible domain $\mathcal{D}$, FW solves the linear optimization subproblem \eqref{eq:sub}. The linearity of the objective \eqref{eq:sub} implies that a solution $\mat{S}^{(k)}$ always lies at an extremal point of $\mathcal{D}$. This leads to the interpretation of FW as a greedy algorithm which adds an extremal point to the current solution at each iteration \citep{Clarkson2010}. In other words, $\mat{M}^{(k)}$ can be written as a sparse convex combination of extremal points:
\begin{equation}
\label{eq:gen_decompo}
\mat{M}^{(k)}=\sum_{\mat{S}^{(k)}\in\mathcal{S}^{(k)}}\alpha^{(k)}_{\mat{S}^{(k)}}\mat{S}^{(k)},\quad\quad \text{ where } \sum_{\mat{S}^{(k)}\in\mathcal{S}^{(k)}}\alpha^{(k)}_{\mat{S}^{(k)}} = 1 \text{ and } \alpha^{(k)}_{\mat{S}^{(k)}}\geq 0,
\end{equation}
where $\mathcal{S}^{(k)}$ denotes the set of ``active'' extremal points that have been added up to iteration $k$.
When the extremal points of $\mathcal{D}$ have specific structure (such as sparsity, or low-rankness), this structure can be leveraged to compute a solution of \eqref{eq:sub} much more efficiently than the projection operator, see \citet{Jaggi2011a,Jaggi2013} for compelling examples.

A drawback of the standard FW algorithm is that ``removing'' an extremal point $\mat{S}^{(k)}$ from the current iterate (or significantly reducing its weight $\alpha^{(k)}_{\mat{S}^{(k)}}$) can only be done indirectly by adding (increasing the weight of) other extremal points. The variant of FW with \emph{away steps} \citep{Guelat1986} addresses this issue by allowing the algorithm to choose between adding a new extremal point (forward step) or reducing the weight of an existing one (away step), as shown in Algorithm~\ref{alg:gen_away_fw}. This can lead to sparser solutions \citep{Guelat1986,Clarkson2010,Jaggi2011a} and faster convergence in some cases \citep{Guelat1986,Lacoste-Julien2015a}.

In the present work, we will introduce a FW algorithm with away steps to efficiently perform similarity learning for high-dimensional sparse data. One of our key ingredients will be the design of a feasible domain with appropriate sparsity structure.

\begin{algorithm}[t]
\caption{Frank-Wolfe algorithm with away steps}
\label{alg:gen_away_fw}
\begin{algorithmic}[0]
\STATE \textbf{Input:} Initial point $\mat{M}^{(0)} \in \mathcal{D}$
\FOR{$k = 0,1,2,\dots$}
    \STATE $\mat{S}^{(k)}_F \leftarrow \argmin_{\mat{S}\in\mathcal{D}}\innerp{\mat{S},\nabla f(\mat{M}^{(k)})}$, $\mat{D}_{F}^{(k)} = \mat{S}^{(k)}_F - \mat{M}^{(k)}$ \hfill{\it\footnotesize // forward direction}
    \STATE $\mat{S}^{(k)}_A \leftarrow \argmax_{\mat{S}\in\mathcal{S}^{(k)}}\innerp{\mat{S},\nabla f(\mat{M}^{(k)})}$, $\mat{D}_{A}^{(k)} = \mat{M}^{(k)} - \mat{S}^{(k)}_A$ \hfill{\it\footnotesize // away direction}
    \IF{$\innerp{\mat{D}_{F}^{(k)},\nabla f(\mat{M}^{(k)})} \leq \innerp{\mat{D}_{A}^{(k)},\nabla f(\mat{M}^{(k)})}$}
\STATE $\mat{D}^{(k)} \leftarrow \mat{D}_{F}^{(k)}$ and $\gamma_{\max} \leftarrow 1$\hfill{\it\footnotesize // choose forward step~~}
\ELSE
\STATE $\mat{D}^{(k)} \leftarrow \mat{D}_{A}^{(k)}$ and $\gamma_{\max} \leftarrow \alpha^{(k)}_{\mat{S}^{(k)}_A} / (1-\alpha^{(k)}_{\mat{S}^{(k)}_A})$\hfill{\it\footnotesize // choose away step~~}
\ENDIF
    \STATE $\gamma^{(k)} \leftarrow \frac{2}{k+2}$ (or determined by line search)
    \STATE $\mat{M}^{(k+1)} \leftarrow \mat{M}^{(k)} + \gamma^{(k)}\mat{D}^{(k)}$
\ENDFOR
\end{algorithmic}
\end{algorithm}

%% file: approach.tex

\section{Proposed Approach}
\label{sApproach}


This section introduces \hdsl (High-Dimensional Similarity Learning), the approach proposed in this paper. We first describe our problem formulation (Section~\ref{sec:form}), then derive and analyze an efficient FW algorithm to solve it in Section~\ref{sec:algo}.

\subsection{Problem Formulation}
\label{sec:form}

In this work, our goal is to learn a similarity function for high-dimensional sparse data. We assume the data points lie in some space $\mathcal{X}\subseteq\mathbb{R}^d$, where $d$ is large ($d >10^4$) but points are $s$-sparse on average ($s \ll d$). In other words, their number of nonzero entries is typically much smaller than $d$.
We focus on learning a similarity function $S_{\mat{M}}:\mathcal{X}\times\mathcal{X} \rightarrow \mathbb{R}$ of the form
$$S_{\mat{M}}(\vct{x},\vct{x}') = \vct{x}\T\mat{M}\vct{x}' = \innerp{\vct{x}\vct{x}'^\top,\mat{M}},$$
where $\mat{M}\in\mathbb{R}^{d\times d}$ and $\innerp{\cdot,\cdot}$ denotes the Frobenius inner product. Notice that for any $\mat{M}$, $S_{\mat{M}}$ can be computed in $O(s^2)$ time on average if data points are stored in a sparse format.

\paragraph{Feasible domain} We will derive an algorithm to learn a very sparse $\mat{M}$ with time and memory requirements that depend on $s$ but not on $d$. To this end, given a scale $\lambda>0$ which will play the role of a regularization parameter, we parameterize $\mat{M}$ as a convex combination of rank-one, 4-sparse $d\times d$ bases:
\begin{equation*}
\mat{M}\in\mathcal{D}_\lambda = \conv(\mathcal{B}_\lambda),\quad\quad\text{ with } \mathcal{B}_\lambda = \bigcup_{ij} \left\{\mat{P}_\lambda^{(ij)}, \mat{N}_\lambda^{(ij)}\right\},
\end{equation*}
where for any pair of features $i,j\in\{1,\dots,d\}$, $i\neq j$,
{\arraycolsep=1.4pt\def\arraystretch{0.5}
\begin{align*}
\mat{P}_\lambda^{(ij)} = \lambda(\boldsymbol{e}_i+\boldsymbol{e}_j)(\boldsymbol{e}_i+\boldsymbol{e}_j)^T &= \left(\begin{array}{ccccc} \cdot & \cdot & \cdot & \cdot & \cdot\\ \cdot & \lambda & \cdot & \lambda & \cdot\\\cdot & \cdot & \cdot & \cdot & \cdot\\\cdot & \lambda & \cdot & \lambda & \cdot\\ \cdot & \cdot & \cdot & \cdot & \cdot\end{array}\right),\\
\mat{N}_\lambda^{(ij)} = \lambda(\boldsymbol{e}_i-\boldsymbol{e}_j)(\boldsymbol{e}_i-\boldsymbol{e}_j)^T &= \left(\begin{array}{ccccc} \cdot & \cdot & \cdot & \cdot & \cdot\\ \cdot & \lambda & \cdot & -\lambda & \cdot\\\cdot & \cdot & \cdot & \cdot & \cdot\\\cdot & -\lambda & \cdot & \lambda & \cdot\\ \cdot & \cdot & \cdot & \cdot & \cdot\end{array}\right).
\end{align*}
}
The use of such sparse matrices was first suggested by \citet{Jaggi2011a}. Besides the fact that they are instrumental to the efficiency of our algorithm (see Section~\ref{sec:algo}), we give some additional motivation for their use in the context of similarity learning.

First, any $\mat{M}\in\mathcal{D}_\lambda$ is a convex combination of symmetric PSD matrices and is thus also symmetric PSD. Unlike many metric learning algorithms, we thus avoid the $O(d^3)$ cost of projecting onto the PSD cone. Constraining $\mat{M}$ to be symmetric PSD provides useful regularization to prevent overfitting \citep{Chechik2009} and ensures that $S_{\mat{M}}$ can be interpreted as a dot product after a linear transformation of the inputs:
$$S_{\mat{M}}(\vct{x},\vct{x}') = \vct{x}\T\mat{M}\vct{x}' = (\mat{L}\vct{x})\T(\mat{L}\vct{x}'),$$
where $\mat{M} = \mat{L}\mat{L}\T$ with $\mat{L}\in\mathbb{R}^{d\times k}$. Because the bases in $\mathcal{B}_\lambda$ are rank-one, the dimensionality $k$ of the transformed space is at most the number of bases composing $\mat{M}$.

Second, each basis operates on two features only. In particular, $S_{\mat{P}_\lambda^{(ij)}}(\vct{x},\vct{x}') = \lambda(x_ix_i' + x_jx_j' + x_ix_j' + x_jx_i')$ assigns a higher similarity score when feature $i$ appears jointly in $\vct{x}$ and $\vct{x}'$ (likewise for $j$), as well as when feature $i$ in $\vct{x}$ and feature $j$ in $\vct{y}$ co-occur (and vice versa). Conversely, $S_{\mat{N}_\lambda^{(ij)}}$ penalizes the cross-occurrences of features $i$ and $j$. In the context of text data represented as bags-of-words (or other count data), the semantic behind the bases in $\mathcal{B}_\lambda$ is quite natural: they can be intuitively thought of as encoding the fact that a term $i$ or $j$ present in both documents makes them more similar, and that two terms $i$ and $j$ are associated with the same/different class or topic.

Optimizing over the convex hull $\mathcal{D}_\lambda$ of $\mathcal{B}_\lambda$ will allow us to easily control the number of active features, thereby learning a very compact representation with efficient similarity computations.

\paragraph{Optimization problem} We now describe the optimization problem to learn the similarity parameters. Following previous work \citep[see for instance][]{Schultz2003,Weinberger2009,Chechik2009}, our training data consists of weak supervision in the form of triplet constraints:
$$\mathcal{T} = \left\{\vct{x}_t \text{ should be more similar to }\vct{y}_t \text{ than to } \vct{z}_t\right\}_{t=1}^T.$$
Such constraints can be built from a labeled training sample (see Section~\ref{sGen}), provided directly by domain experts or crowdsourcing campaign, or obtained through implicit feedback such as clicks on search engine results. For notational convenience, we denote $\mat{A}^t = \vct{x}_t(\vct{y}_t-\vct{z}_t)\T\in\mathbb{R}^{d\times d}$ for each constraint $t=1,\dots,T$ so that we can concisely write $S_{\mat{M}}(\vct{x}_t,\vct{y}_t) - S_{\mat{M}}(\vct{x}_t,\vct{z}_t) = \innerp{\mat{A}^t,\mat{M}}$.
We measure the degree of violation of each constraint $t$ with the smoothed hinge loss $\ell:\mathbb{R}\rightarrow\mathbb{R}^+$ defined as
$$\ell\left(\innerp{\mat{A}^t,\mat{M}}\right) = \left\{ \begin{array}{ll} 0 & \text{if } \innerp{\mat{A}^t,\mat{M}} \geq 1\\\frac{1}{2} - \innerp{\mat{A}^t,\mat{M}} & \text{if } \innerp{\mat{A}^t,\mat{M}} \leq 0\\\frac{1}{2} \left(1-\innerp{\mat{A}^t,\mat{M}}\right)^2& \text{otherwise}\end{array}\right..$$
This convex loss is a continuously differentiable version of the standard hinge loss which tries to enforce a margin constraint of the form $S_{\mat{M}}(\vct{x}_t,\vct{y}_t) \geq S_{\mat{M}}(\vct{x}_t,\vct{z}_t) +1$. When this constraint is satisfied, the value of the loss is zero. On the other hand, when the margin is negative, i.e. $S_{\mat{M}}(\vct{x}_t,\vct{y}_t) \leq S_{\mat{M}}(\vct{x}_t,\vct{z}_t)$, the penalty is linear in the margin violation. A quadratic interpolation is used to bridge between these two cases to ensure that the loss is differentiable everywhere.

\begin{rem}[Choice of loss]
One may use any other convex and continuously differentiable loss function in our framework, such as the squared hinge loss, the logistic loss or the exponential loss.
\end{rem}

Given $\lambda>0$, our similarity learning formulation aims at finding the matrix $\mat{M}\in\mathcal{D}_\lambda$ that minimizes the average margin penalty (as measured by $\ell$) over the triplet constraints in $\mathcal{T}$:

\begin{equation}
\label{eq:form}
\min_{\mat{M}\in\mathbb{R}^{d\times d}} \quad f(\mat{M}) = \frac{1}{T}\sum_{t=1}^T \ell\left(\innerp{\mat{A}^t,\mat{M}}\right)\quad\text{s.t.}\quad \mat{M}\in\mathcal{D}_\lambda.
\end{equation}

Due to the convexity of the smoothed hinge loss, \eqref{eq:form} involves minimizing a convex function over the convex domain $\mathcal{D}_\lambda$. Note that the gradient of the objective is given by
\begin{equation}
\begin{gathered}
\label{eq:gradient}
\nabla f(\mat{M}) = \frac{1}{T}\sum_{t=1}^T \mat{G}^t(\mat{M}),\\
\text{with }\mat{G}^t(\mat{M}) = \left\{ \begin{array}{ll} \mat{0} & \text{if } \innerp{\mat{A}^t,\mat{M}} \geq 1\\-\mat{A}^t&\text{if }\innerp{\mat{A}^t,\mat{M}} \leq 0\\\left(\innerp{\mat{A}^t,\mat{M}} - 1\right)\mat{A}^t& \text{otherwise}\end{array}\right..
\end{gathered}
\end{equation}

In the next section, we propose a greedy algorithm to efficiently find sparse approximate solutions to this problem.

\subsection{Algorithm}
\label{sec:algo}

\subsubsection{Exact Frank-Wolfe Algorithm}

\begin{algorithm}[t]
\caption{Frank Wolfe algorithm for problem \eqref{eq:form}}
\label{alg:fw}
\begin{algorithmic}[1]
\STATE initialize $\mat{M}^{(0)}$ to an arbitrary $\mat{B}\in\mathcal{B}_\lambda$
\FOR{$k = 0,1,2,\dots$}
\STATE $\mat{B}_{F}^{(k)} \leftarrow \argmin_{\mat{B}\in\mathcal{B}_\lambda} \innerp{\mat{B},\nabla f(\mat{M}^{(k)})}$, $\mat{D}_{F}^{(k)} \leftarrow \mat{B}_{F}^{(k)} - \mat{M}^{(k)}$ \hfill{\it\footnotesize // forward dir.}
\STATE $\mat{B}_{A}^{(k)} \leftarrow \argmax_{\mat{B}\in\mathcal{S}^{(k)}} \innerp{\mat{B},\nabla f(\mat{M}^{(k)})}$, $\mat{D}_{A}^{(k)} \leftarrow  \mat{M}^{(k)} - \mat{B}_{A}^{(k)}$ \hfill{\it\footnotesize // away dir.}
\IF{$\innerp{\mat{D}_{F}^{(k)},\nabla f(\mat{M}^{(k)})} \leq \innerp{\mat{D}_{A}^{(k)},\nabla f(\mat{M}^{(k)})}$}
\STATE $\mat{D}^{(k)} \leftarrow \mat{D}_{F}^{(k)}$ and $\gamma_{\max} \leftarrow 1$\hfill{\it\footnotesize // choose forward step~~}
\ELSE
\STATE $\mat{D}^{(k)} \leftarrow \mat{D}_{A}^{(k)}$ and $\gamma_{\max} \leftarrow \alpha^{(k)}_{\mat{B}_{A}^{(k)}} / (1-\alpha^{(k)}_{\mat{B}_{A}^{(k)}})$\hfill{\it\footnotesize // choose away step~~}
\ENDIF
\STATE $\gamma^{(k)} \leftarrow \argmin_{\gamma\in[0,\gamma_{\max}]} f(\mat{M}^{(k)}+\gamma \vct{D}^{(k)})$\hfill{\it\footnotesize // perform line search}
\STATE $\mat{M}^{(k+1)} \leftarrow \mat{M}^{(k)} + \gamma^{(k)} \mat{D}^{(k)}$ \hfill{\it\footnotesize // update iterate towards direction~~}
\ENDFOR
\end{algorithmic}
\end{algorithm}

We propose to use a Frank-Wolfe algorithm with away steps (see Section~\ref{sec:related_fw}) to learn the similarity. We will exploit the fact that in our formulation \eqref{eq:form}, the extremal points (vertices) of the feasible domain $\mathcal{D}_\lambda$ are the elements of $\mathcal{B}_\lambda$ and have special structure.
Our algorithm is shown in Algorithm~\ref{alg:fw}. 
During the course of the algorithm, we explicitly maintain a representation of each iterate $\mat{M}^{(k)}$ as a convex combination of basis elements as previously discussed in Section~\ref{sec:related_fw}:
$$\mat{M}^{(k)} = \sum_{\mat{B}\in\mathcal{B}_\lambda}\alpha^{(k)}_{\mat{B}}\mat{B},\quad\quad \text{ where } \sum_{\mat{B}\in\mathcal{B}_\lambda}\alpha^{(k)}_{\mat{B}} = 1 \text{ and } \alpha^{(k)}_{\mat{B}}\geq 0.$$
We denote the set of active basis elements in $\mat{M}^{(k)}$ as $\mathcal{S}^{(k)} = \{\mat{B}\in\mathcal{B}_\lambda : \alpha^{(k)}_{\mat{B}} > 0\}$.
The algorithm goes as follows. We initialize $\mat{M}^{(0)}$ to a random basis element. Then, at each iteration, we greedily choose between moving towards a (possibly) new basis (forward step) or reducing the weight of an active one (away step). The extent of the step is determined by line search. As a result, Algorithm~\ref{alg:fw} adds only one basis (at most 2 new features) at each iteration, which provides a convenient way to control the number of active features and maintains a compact representation of $\mat{M}^{(k)}$ for a memory cost of $O(k)$. Furthermore, away steps provide a way to reduce the importance of a potentially ``bad'' basis element added at an earlier iteration (or even remove it completely when $\gamma^{(k)} = \gamma_{\max}$). Recall that throughout the execution of the FW algorithm, all iterates $\mat{M}^{(k)}$ remain convex combinations of basis elements and are thus feasible. The following proposition shows that the iterates of Algorithm~\ref{alg:fw} converge to an optimal solution of \eqref{eq:form} with a rate of $O(1/k)$.

%

\begin{prop}
\label{prop:converge}
Let $\lambda>0$, $\mat{M}^*$ be an optimal solution to \eqref{eq:form} and $L = \frac{1}{T}\sum_{t=1}^T\|\mat{A}^t\|_F^2$. At any iteration $k\geq 1$ of Algorithm~\ref{alg:fw}, the iterate $\mat{M}^{(k)}\in\mathcal{D}_\lambda$ satisfies $f(\mat{M}^{(k)})-f(\mat{M}^*) \leq 16L\lambda^2/(k+2)$. Furthermore, it has at most rank $k+1$ with $4(k+1)$ nonzero entries, and uses at most $2(k+1)$ distinct features.
\end{prop}
\begin{proof}
We first show that $\nabla f$ is $L$-Lipschitz continuous on $\mathcal{D}_\lambda$ with respect to the Frobenius norm, i.e. for any $\mat{M}_1,\mat{M}_2\in\mathcal{D}_\lambda$,
\begin{equation}
\label{eq:lip}
\|\nabla f(\mat{M}_1)-\nabla f(\mat{M}_2)\|_F \leq L\|\mat{M}_1-\mat{M}_2\|_F
\end{equation}
for some $L\geq 0$. Note that
\begin{align*}
\|\nabla f(\mat{M}_1)-\nabla f(\mat{M}_2)\|_F &= \left\|\frac{1}{T}\sum_{t=1}^T\mat{G}^t(\mat{M}_1) - \frac{1}{T}\sum_{t=1}^T\mat{G}^t(\mat{M}_2)\right\|_F\\
&\leq\frac{1}{T}\sum_{t=1}^T\left\|\mat{G}^t(\mat{M}_1) - \mat{G}^t(\mat{M}_2)\right\|_F.
\end{align*}
Let $t\in\{1,\dots,T\}$. We will now bound $\Delta_t=\left\|\mat{G}^t(\mat{M}_1) - \mat{G}^t(\mat{M}_2)\right\|_F$ for any $\mat{M}_1,\mat{M}_2\in\mathcal{D}_\lambda$. The form of the gradient \eqref{eq:gradient} requires to consider several cases:
\begin{itemize}
\item[(i)] If $\innerp{\mat{A}^t,\mat{M}_1} \geq 1$ and $\innerp{\mat{A}^t,\mat{M}_2} \geq 1$, we have $\Delta_t=0$.
\item[(ii)] If $\innerp{\mat{A}^t,\mat{M}_1} \leq 0$ and $\innerp{\mat{A}^t,\mat{M}_2} \leq 0$, we have $\Delta_t=0$.
\item[(iii)] If $0 < \innerp{\mat{A}^t,\mat{M}_1} < 1$ and $0<\innerp{\mat{A}^t,\mat{M}_2}<1$, we have:
\begin{align*}
\Delta_t &= \| \innerp{\mat{A}^t,\mat{M}_1-\mat{M}_2}\mat{A}^t \|_F = \| \mat{A}^t \|_F|\innerp{\mat{A}^t,\mat{M}_1-\mat{M}_2}|\\
&\leq \|\mat{A}^t\|_F^2 \|\mat{M}_1-\mat{M}_2\|_F.
\end{align*}
\item[(iv)] If $\innerp{\mat{A}^t,\mat{M}_1} \geq 1$ and $\innerp{\mat{A}^t,\mat{M}_2} \leq 0$, we have
\begin{align*}
\Delta_t &= \| \mat{A}^t \|_F \leq \| \mat{A}^t \|_F | \innerp{\mat{A}^t,\mat{M}_1-\mat{M}_2} | \leq \|\mat{A}^t\|_F^2 \|\mat{M}_1-\mat{M}_2\|_F.
\end{align*}
\item[(v)] If $\innerp{\mat{A}^t,\mat{M}_1} \geq 1$ and $0<\innerp{\mat{A}^t,\mat{M}_2}<1$, we have:
\begin{align*}
\Delta_t &= \| (\innerp{\mat{A}^t,\mat{M}_2}-1)\mat{A}^t \|_F = \|\mat{A}^t\|_F(1-\innerp{\mat{A}^t,\mat{M}_2})\\
&\leq \|\mat{A}^t\|_F(1-\innerp{\mat{A}^t,\mat{M}_2}) + \|\mat{A}^t\|_F(\innerp{\mat{A}^t,\mat{M}_1}-1)\\
&= \|\mat{A}^t\|_F\innerp{\mat{A}^t,\mat{M}_1-\mat{M}_2} \leq \|\mat{A}^t\|_F^2 \|\mat{M}_1-\mat{M}_2\|_F.
\end{align*}
\item[(vi)] If $\innerp{\mat{A}^t,\mat{M}_1} \leq 0$ and $0<\innerp{\mat{A}^t,\mat{M}_2}<1$, we have:
\begin{align*}
\Delta_t &= \| -\mat{A}^t - (\innerp{\mat{A}^t,\mat{M}_2}-1)\mat{A}^t \|_F = \| \mat{A}^t \innerp{\mat{A}^t,\mat{M}_2}\|_F = \| \mat{A}^t\|_F \innerp{\mat{A}^t,\mat{M}_2}\\
&\leq \| \mat{A}^t\|_F \innerp{\mat{A}^t,\mat{M}_2} - \| \mat{A}^t\|_F \innerp{\mat{A}^t,\mat{M}_1}\\
&= \|\mat{A}^t\|_F\innerp{\mat{A}^t,\mat{M}_2-\mat{M}_1} \leq \|\mat{A}^t\|_F^2 \|\mat{M}_1-\mat{M}_2\|_F.
\end{align*}
\end{itemize}
The remaining cases are also bounded by $\|\mat{A}^t\|_F^2 \|\mat{M}_1-\mat{M}_2\|_F$ by symmetry to cases (iv)-(v)-(vi). Hence $\nabla f$ is $L$-Lipschitz continuous with $L=\|\mat{A}^t\|_F^2$.

It is easy to see that $\diam_{\|\cdot\|_F}(\mathcal{D}_\lambda)=\sqrt{8}\lambda$. The convergence rate then follows from the general analysis of the FW algorithm \citep{Jaggi2013}.

The second part of the proposition follows directly from the structure of the bases and the greedy nature of the algorithm.
\end{proof}

Note that the optimality gap in Proposition~\ref{prop:converge} is independent of $d$. Indeed, $\mat{A}^t$ has $O(s^2)$ nonzero entries on average, hence the term $\|\mat{A}^t\|_F^2$ in the Lipschitz constant $L$ can be bounded by $s^2 \|\mat{A}^t\|_\infty$, where $\|\mat{A}\|_\infty = \max_{i,j=1}^d|A_{i,j}|$. This means that Algorithm~\ref{alg:fw} is able to find a good approximate solution based on a small number of features in only a few iterations, which is very appealing in the high-dimensional setting we consider.

\subsubsection{Complexity Analysis}

We now analyze the time and memory complexity of Algorithm~\ref{alg:fw}.
The form of the gradient \eqref{eq:gradient} along with the structure of the algorithm's updates are crucial to its efficiency. Since $\mat{M}^{(k+1)}$ is a convex combination of $\mat{M}^{(k)}$ and a 4-sparse matrix $\mat{B}^{(k)}$, we can efficiently compute most of the quantities of interest through careful book-keeping.

In particular, storing $\mat{M}^{(k)}$ at iteration $k$ requires $O(k)$ memory. We can also recursively compute $\innerp{\mat{A}^t,\mat{M}^{(k+1)}}$ for all constraints in only $O(T)$ time and $O(T)$ memory based on $\innerp{\mat{A}^t,\mat{M}^{(k)}}$ and $\innerp{\mat{A}^t,\mat{B}^{(k)}}$. This allows us, for instance, to efficiently compute the objective value as well as to identify the set of satisfied constraints (those with $\innerp{\mat{A}^t,\mat{M}^{(k)}} \geq 1$) which are ignored in the computation of the gradient.
Finding the away direction at iteration $k$ can be done in $O(Tk)$ time.
For the line search, we use a bisection algorithm to find a root of the gradient of the 1-dimensional function of $\gamma$, which only depends on $\innerp{\mat{A}^t,\mat{M}^{(k)}}$ and $\innerp{\mat{A}^t,\mat{B}^{(k)}}$, both of which are readily available. Its time complexity is $O(T\log \frac{1}{\epsilon})$ where $\epsilon$ is the precision of the line-search, with a memory cost of $O(1)$.

The bottleneck is to find the forward direction. 
Indeed, sequentially considering each basis element is intractable as it takes $O(Td^2)$ time. A more efficient strategy is to sequentially consider each constraint, which requires $O(Ts^2)$ time and $O(Ts^2)$ memory. The overall iteration complexity of Algorithm~\ref{alg:fw} is given in Table~\ref{tab:complexity}.

\subsubsection{Approximate Forward Step}
\label{sec:approx}

\begin{table}[t]
\centering
\def\arraystretch{1.3}
\begin{tabular}{|c||c|c|}
\hline Variant & Time complexity & Memory complexity \\
\hline\hline Exact (Algorithm~\ref{alg:fw}) & $\tilde{O}(Ts^2+Tk)$ & $\tilde{O}(Ts^2+k)$\\
\hline Mini-batch & $\tilde{O}(Ms^2+Tk)$ & $\tilde{O}(T+Ms^2 + k)$\\
\hline Mini-batch + heuristic & $\tilde{O}(Ms+Tk)$ & $\tilde{O}(T+Ms + k)$\\
\hline
\end{tabular}
\caption{Complexity of iteration $k$ (ignoring logarithmic factors) for different variants of the algorithm.}
\label{tab:complexity}
\end{table}

Finding the forward direction can be expensive when $T$ and $s$ are both large. We propose two strategies to alleviate this cost by finding an approximately optimal basis (see Table~\ref{tab:complexity} for iteration complexity).

\paragraph{Mini-batch approximation} Instead of finding the forward and away directions based on the full gradient at each iteration, we can estimate it on a mini-batch of $M\ll T$ constraints drawn uniformly at random (without replacement). The complexity of finding the forward direction is thus reduced to $O(Ms^2)$ time and $O(Ms^2)$ memory. Consider the deviation between the ``value'' of any basis element $\mat{B}\in\mathcal{B}_\lambda$ on the full set of constraints and its estimation on the mini-batch, namely
\begin{equation}
\label{eq:batch}
\Big|\frac{1}{M}\sum_{t\in\mathcal{M}}\innerp{\mat{B},\mat{G}_t}-\frac{1}{T}\sum_{t=1}^T\innerp{\mat{B},\mat{G}_t}\Big|,
\end{equation}
where $\mathcal{M}$ is the set of $M$ constraint indices drawn uniformly and without replacement from the set $\{1,\dots,T\}$.
Under mild assumptions, concentration bounds such as Hoeffding's inequality for sampling without replacement \citep{Serfling1974,Bardenet2015a} can be used to show that the probability of \eqref{eq:batch} being larger than some constant decreases exponentially fast with $M$. The FW algorithm is known to be robust to inexact gradients, and convergence guarantees similar to Proposition~\ref{prop:converge} can be obtained directly from \citep{Jaggi2013,Freund2013}.

\paragraph{Fast heuristic} To avoid the quadratic dependence on $s$, we propose to use the following heuristic to find a good forward basis. We first pick a feature $i\in[d]$ uniformly at random, and solve the linear problem over the restricted set $\bigcup_{j} \{\mat{P}_\lambda^{(ij)}, \mat{N}_\lambda^{(ij)}\}$. We then solve it again over the set $\bigcup_{k} \{\mat{P}_\lambda^{(kj)}, \mat{N}_\lambda^{(kj)}\}$ and use the resulting basis for the forward direction. This can be done in only $O(Ms)$ time and $O(Ms)$ memory and gives good performance in practice, as we shall see in Section~\ref{sExp}.


%
%
%

%% file: generalization.tex

\section{Generalization Analysis}
\label{sGen}

In this section, we derive generalization bounds for the proposed method. Our main goal is to give a theoretical justification of our approach, in particular by (i) showing that our choice of feasible domain $\mathcal{D}_\lambda$ helps to reduce overfitting in high dimensions, and (ii) showing that the proposed greedy Frank-Wolfe algorithm provides a simple way to balance between optimization and generalization errors through early stopping.

\subsection{Setup and Notations}

As in previous work on generalization bounds for metric learning, we consider the supervised learning setting where the training sample is a set of labeled points $S = \{\vct{z}_i=(\vct{x}_i,y_i)\}_{i=1}^n$ drawn i.i.d. from a probability distribution $\mu$ over the space $\mathcal{Z}=\mathcal{X}\times\mathcal{Y}$, where $\mathcal{X}\subseteq\mathbb{R}^d$ and $\mathcal{Y}=\{1,\dots,C\}$ is the label set.
We assume that $B_{\mathcal{X}}= \sup_{\vct{x},\vct{x}',\vct{x}''\in\mathcal{X}}\|\vct{x}(\vct{x}'-\vct{x}'')^T\|$ is bounded for some convenient matrix norm $\|\cdot\|$.

For simplicity, we assume that the univariate loss function $\ell : \mathbb{R} \rightarrow \mathbb{R}^+ $ is 1-Lipschitz, which is the case for the smoothed hinge loss used in our algorithm. Given a triplet $(\vct{z},\vct{z}',\vct{z}'')\in\mathcal{Z}^3$, we say that it is \emph{admissible} if $y=y'\neq y''$. Since we only want to consider admissible triplets, we will use the triplet-wise loss function $L_{\mat{M}}(\vct{z},\vct{z}',\vct{z}'') = \mathbbm{I}[y= y'\neq y'']\cdot\ell(\innerp{\vct{x}(\vct{x}'-\vct{x}'')\T,\mat{M}})$ indexed by $\mat{M}\in\mathcal{D}_\lambda$, which is equal to zero for non-admissible triplets.

Given a matrix $\mat{M}\in\mathcal{D}_\lambda$, we define its empirical risk associated on the training set $S$ as follows:
\begin{equation}
\label{eq:emprisk}
\mathcal{L}_S(\mat{M}) = \frac{1}{n(n-1)(n-2)}\sum_{i\neq j\neq k} L_{\mat{M}}(\vct{z}_i,\vct{z}_j,\vct{z}_k).
\end{equation}
Similarly, its expected risk is defined as
\begin{equation}
\label{eq:exprisk}
\mathcal{L}(\mat{M}) = \expec_{\vct{z},\vct{z}',\vct{z}''\sim\mu} \left[L_{\mat{M}}(\vct{z},\vct{z}',\vct{z}'')\right].
\end{equation}
In contrast to the standard supervised classification setting, note that the empirical risk \eqref{eq:emprisk} takes the form of an \emph{average of dependent terms} known as a $U$-statistic \citep{Lee1990a}.

From our feasible domain $\mathcal{D}_\lambda = \conv(\mathcal{B}_\lambda)$, we can define a sequence of nested sets as follows:
\begin{equation}
\mathcal{D}_\lambda^{(k)}=\left\{\sum_{i=1}^k\alpha_i\mat{B}_i: \mat{B}_i \in \mathcal{B}_\lambda, \alpha_i\geq 0, \sum_{i=1}^k\alpha_i = 1\right\},\quad k=1,\dots,2d(d-1).
\end{equation}

In other words, $\mathcal{D}_\lambda^{(k)}$ consists of all $d\times d$ matrices which can be decomposed as a convex combination of at most $k$ elements of the basis set $\mathcal{B}_\lambda$. Clearly, we have $\mathcal{D}_\lambda^{(1)}\subset\mathcal{D}_\lambda^{(2)}\subset\dots\subset\mathcal{D}_\lambda^{(2d(d-1))} = \mathcal{D}_\lambda$. Note also that since $\ell$ is 1-Lipschitz, by Holder's inequality we have $\forall k$:
\begin{align} 
\sup_{\vct{z},\vct{z}',\vct{z}''\in\mathcal{Z},\mat{M}\in\mathcal{D}_\lambda^{(k)}} |L_{\mat{M}}(\vct{z},\vct{z}',\vct{z}'')| &\leq \sup_{\vct{x},\vct{x}',\vct{x}''\in\mathcal{X},\mat{M}\in\mathcal{D}_\lambda^{(k)}} |\ell ( \innerp{\vct{x} (\vct{x}' -\vct{x}'')^T, \mat{M}})|\nonumber\\
&\leq B_{\mathcal{X}}\sup_{\mat{M}\in\mathcal{D}_\lambda^{(k)}}\|\mat{M}\|_*,\label{eq:holder}
\end{align}
where $\|\cdot\|_*$ is the dual norm of $\|\cdot\|$.

In the following, we derive theoretical results that take advantage of the structural properties of our algorithm, namely that the matrix $\mat{M}^{(k)}$ returned after $k\geq 1$ iterations of Algorithm~\ref{alg:fw} belongs to $\mathcal{D}_\lambda^{(k)}$. We first bound the Rademacher complexity of $\mathcal{D}_\lambda^{(k)}$ and derive bounds on the maximal deviation between $\mathcal{L}(\mat{M})$ and $\mathcal{L}_S(\mat{M})$ for any $\mat{M}\in\mathcal{D}_\lambda^{(k)}$. We then use these results to derive bounds on the excess risk $\mathcal{L}(\mat{M}^{(k)}) - \mathcal{L}(\mat{M}^*)$, where $\mat{M}^*\in \argmin_{\mat{M}\in\mathcal{D}_\lambda} \mathcal{L}(\mat{M})$ is the expected risk minimizer. All proofs can be found in the appendix.

\subsection{Main Results}

We first characterize the Rademacher complexity of the loss functions indexed by elements of $\mathcal{D}_\lambda^{(k)}$. Given $k\in\{1,\dots,2d(d-1)\}$, consider the family $\mathcal{F}^{(k)} = \{L_{\mat{M}} : \mat{M}\in\mathcal{D}_\lambda^{(k)}\}$ of functions  mapping from $\mathcal{Z}^3$ to $\mathbb{R}^+$. We will consider the following definition of the Rademacher complexity of $\mathcal{F}^{(k)}$ with respect to distribution $\mu$ and sample size $n\geq 3$, adapted from \citep{Clemencon2008a,Cao2012a}:
\begin{equation}
\label{eq:rademacher}
R_n\left(\mathcal{F}^{(k)}\right) = \mathbb{E}_{\vct{\sigma}, S\sim\mu^n}\bigg[\sup_{\mat{M}\in\mathcal{D}_\lambda^{(k)}}\frac{1}{\lfloor n/3 \rfloor}\sum_{i=1}^{\lfloor n/3 \rfloor} \sigma_i  L_{\mat{M}}(\vct{z}_{i},\vct{z}_{i+{\lfloor n/3 \rfloor}},\vct{z}_{i+2\times {\lfloor n/3 \rfloor}})\bigg],
\end{equation}
where $\vct{\sigma}=(\sigma_1,\dots,\sigma_{\lfloor n/3 \rfloor})$ are independent uniform random variables taking values in $\{-1, 1\}$.
The following lemma gives a bound on the above Rademacher complexity.

\begin{lemma}[Bounded Rademacher complexity]
\label{lem:rademacher}
Let $n\geq 3$, $\lambda>0$ and $1\leq k\leq 2d(d-1)$. We have
$$R_n(\mathcal{F}^{(k)}) \leq 8\lambda B_{\mathcal{X}}\sqrt{\frac{2\log k}{\lfloor n/3 \rfloor}}.$$
\end{lemma}
\begin{proof}
See \ref{app:rad}.
\end{proof}

There are two important consequences to Lemma~\ref{lem:rademacher}. First, restricting the set of feasible matrices $\mat{M}$ to $\mathcal{D}_\lambda=\mathcal{D}_\lambda^{(2d(d-1))}$ instead of $\mathbb{R}^{d\times d}$ leads to a Rademacher complexity with a very mild $O(\sqrt{\log{d}})$ dependence in the dimension. This validates our design choice for the feasible domain in the high-dimensional setting we consider. Second, the Rademacher complexity can actually be made independent of $d$ by further restricting the number of bases $k$.

Using this result, we derive a bound for the deviation between the expected risk $\mathcal{L}(\mat{M})$ and the empirical risk $\mathcal{L}_S(\mat{M})$ of any $\mat{M}\in\mathcal{D}_\lambda^{(k)}$.

\begin{theorem}[Maximal deviations]
\label{thm:max_deviations}
Let $S$ be a set of of $n$ points drawn i.i.d. from $\mu$, $\lambda>0$ and $1\leq k\leq 2d(d-1)$. For any $\delta>0$, with probability $1-\delta$ we have
\begin{equation}
\sup_{\mat{M}\in\mathcal{D}_\lambda^{(k)}} [\mathcal{L}(\mat{M}) - \mathcal{L}_S(\mat{M}) ] \leq 16\lambda B_{\mathcal{X}}\sqrt{\frac{2\log k}{\lfloor n/3 \rfloor}}  +  3B_{\mathcal{X}}B_{\mathcal{D}_\lambda^{(k)}}\sqrt{\frac{2\ln{(2/\delta)}}{n}},
\end{equation}
where $B_{\mathcal{D}_\lambda^{(k)}} = \sup_{\mat{M}\in\mathcal{D}_\lambda^{(k)}}\|\mat{M}\|_*$.
\end{theorem}
\begin{proof}
See \ref{app:maxdev}.
\end{proof}

The generalization bounds given by Theorem~\ref{thm:max_deviations} exhibit a standard $O(1\sqrt{n})$ rate. They also confirm that restricting the number $k$ of bases is a good strategy to guard against overfitting when the feature dimension $d$ is high. Interestingly, note that due to the convex hull structure of our basis set, $B_{\mathcal{D}_\lambda^{(k)}} = \sup_{\mat{M}\in\mathcal{D}_\lambda^{(k)}}\|\mat{M}\|_*$ can be easily bounded by a quantity independent of $d$ for any $k\geq 1$ and any dual norm $\|\cdot\|_*$. We thus have complete freedom to choose the primal norm $\|\cdot\|$ so as to make $B_{\mathcal{X}}= \sup_{\vct{x},\vct{x}',\vct{x}''\in\mathcal{X}}\|\vct{x}(\vct{x}'-\vct{x}'')^T\|$ as small as possible. A good choice of primal norm is the infinity norm $\|\mat{A}\|_\infty = \max_{i,j=1}^d|A_{i,j}|$, which is independent of $d$. For instance, if $\mathcal{X}=[0,1]^d$ we have $B_{\mathcal{X}}=1$. The dual norm of the infinity norm being the $L_1$ norm, we then have for any $k\geq 1$:
\begin{equation}
\label{eq:l1_bound}
B_{\mathcal{D}_\lambda^{(k)}} = \sup_{\mat{M}\in\mathcal{D}_\lambda^{(k)}}\|\mat{M}\|_1 = \sup_{\mat{M}\in\mathcal{D}_\lambda^{(k)}}\sum_{i,j=1}^d |M_{i,j}|\leq 4\lambda.
\end{equation}

Theorem~\ref{thm:max_deviations} is directly comparable to the results of \citet{Cao2012a}, who derived generalization bounds for similarity learning under various norm regularizers. Their bounds have a similar form, but exhibit a dependence on the feature dimension $d$ which is at least logarithmic (sometimes even linear, depending on the norm used to regularize the empirical risk). In contrast, our bounds depend logarithmically on $k\ll d$. This offers more flexibility in the high-dimensional setting because $k$ can be directly controlled by stopping our algorithm after $k \ll d$ iterations to guarantee that the output is in $\mathcal{D}_\lambda^{(k)}$. This is highlighted by the following corollary, which combines the generalization bounds of Theorem~\ref{thm:max_deviations} with the $O(1/k)$ convergence rate of our Frank-Wolfe optimization algorithm (Proposition~\ref{prop:converge}).

\begin{cor}[Excess risk bound]
\label{cor:excess_risk}
Let $S$ be a set of $n$ points drawn i.i.d. from $\mu$, $\lambda>0$. Given $k\in\{1,\dots,2d(d-1)\}$, let $\mat{M}^{(k)}$ be the solution returned after $k$ iterations of Algorithm~\ref{alg:fw} applied to the problem $\min_{\mat{M}\in\mathcal{D}_\lambda}\mathcal{L}_S(\mat{M})$, and let $\mat{M}^*\in \argmin_{\mat{M}\in\mathcal{D}_\lambda} \mathcal{L}(\mat{M})$ be the expected risk minimizer over $\mathcal{D}_\lambda$. For any $\delta>0$, with probability $1-\delta$ we have
$$\mathcal{L}(\mat{M}^{(k)}) - \mathcal{L}(\mat{M}^*) \leq \frac{16L\lambda^2}{k+2} + 16\lambda B_{\mathcal{X}}\sqrt{\frac{2\log k}{\lfloor n/3 \rfloor}}  +  5B_{\mathcal{X}}B_{\mathcal{D}_\lambda^{(k)}}\sqrt{\frac{\ln{(4/\delta)}}{n}}.$$
\end{cor}
\begin{proof}
See \ref{app:excess}.
\end{proof}

Corollary~\ref{cor:excess_risk} shows that the excess risk with respect to the expected risk minimizer $\mat{M}^*$ depends on a trade-off between optimization error and complexity of the hypothesis class. Remarkably, this trade-off is ruled by the number $k$ of iterations of the algorithm: as $k$ increases, the optimization error term decreases but the Rademacher complexity terms gets larger. We thus obtain an excess risk bound which adapts to the actual sparsity of the solution output by our algorithm. This is in accordance with our overall goal of reducing overfitting by allowing a strict control on the complexity of the learned similarity, and justifies an early-stopping strategy to achieve a good reduction in empirical risk by selecting the most useful bases while keeping the solution complexity small enough. Again, the excess risk is independent of the feature dimension $d$, suggesting that in the high-dimensional setting it is possible to find sparse solutions with small excess risk. To the best of our knowledge, this is the first result of this nature for metric or similarity learning.

\begin{rem}[Approximation of empirical risk by subsampling]
The empirical risk \eqref{eq:emprisk} is a sum of $O(n^3)$ term, which can be costly to minimize in the large-scale setting. To reduce the computational cost, an alternative to the mini-batch strategy described in Section~\ref{sec:approx} is to randomly subsample $M$ terms of the sum (e.g., uniformly without replacement) and to solve the resulting approximate empirical risk minimization problem. For general problems involving $U$-statistics, \citet{Clemencon2016a} showed that sampling only $M=O(n)$ terms is sufficient to maintain the $O(1/\sqrt{n})$ rate. These arguments can be adapted to our setting to obtain results similar to Theorem~\ref{thm:max_deviations} and Corollary~\ref{cor:excess_risk} for this subsampled empirical risk.
\end{rem}

%% file: exp.tex

\section{Experiments}
\label{sExp}


In this section, we present experiments to evaluate the performance and robustness of \hdsl. In Section~\ref{sec:synth}, we use synthetic data to study the performance of our approach in terms of similarity recovery and generalization in high dimensions in a controlled environment.
Section~\ref{sec:real} evaluates our algorithm against competing approaches on classification and dimensionality reduction using real-world datasets.
Our Matlab code is publicly available on GitHub under
GNU/GPL 3 license.\footnote{\url{https://github.com/bellet/HDSL}} 

\subsection{Experiments on Synthetic Data}
\label{sec:synth}

We first conduct experiments on synthetic datasets in order to address two questions:
\begin{enumerate}
\item Is the algorithm able to recover the ground truth sparse similarity function from (potentially weak) similarity judgments?
\item How well does the algorithm generalize as the dimensionality increases?
\end{enumerate}

\subsubsection{Similarity Recovery}

To investigate the algorithm's ability to recover the underlying similarity, we generate a ground truth similarity metric $M\in \mathbb{R}^{d\times d}$ where $d=2000$. $M$ is constructed as a convex combination of $100$ randomly selected rank-one 4-sparse bases as specified in Section~\ref{sec:form}. The combination coefficients are drawn from a Dirichlet distribution with shape parameter $9$ and scale parameter $0.5$. Without loss of generality, we choose the metric to be block structured by restricting the basis selection from two blocks. This makes the resulting matrix easier to visualize, as show in Figure~\ref{fTrueMetric}.

We then generate $5000$ training samples from the uniform distribution on $[0,1]$ with $2\%$ sparsity. From this sample, we create $30,000$ training triplets $\{(x_1,x_2,x_3)\}$ where $x_1$ is randomly picked and $x_2$ (or $x_3$) is sampled among $x_1$'s top $\alpha$\% similar (or dissimilar) samples as measured by the ground truth metric $M$. The parameter $\alpha$ controls the ``quality'' of the triplet constraints: a larger $\alpha$ leads to less similar (or dissimilar) samples in the triplets, thereby providing a weaker signal about the underlying similarity. We experiment with various $\alpha$ (10\%, 20\%, 25\%, 30\%) to investigate the robustness of \hdsl to the quality of the supervision.
In all our experiments, we use $\lambda=100$.

\begin{figure}[t]
\centering
\subfigure[Feature recovery AUC]{\includegraphics[width=0.45\textwidth]{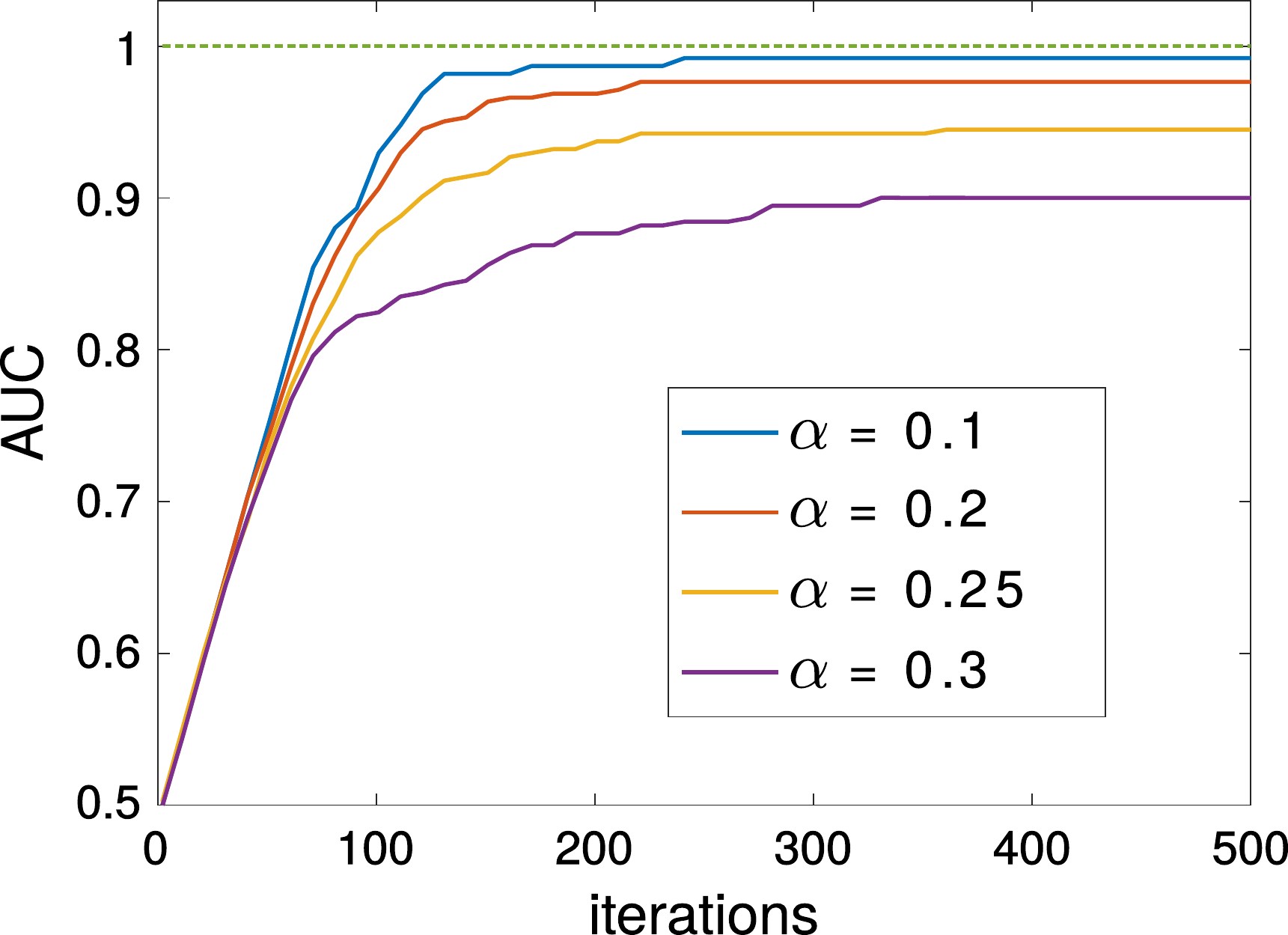}
\label{fAUC}}
\subfigure[Entry recovery AUC]{\includegraphics[width=0.45\textwidth]{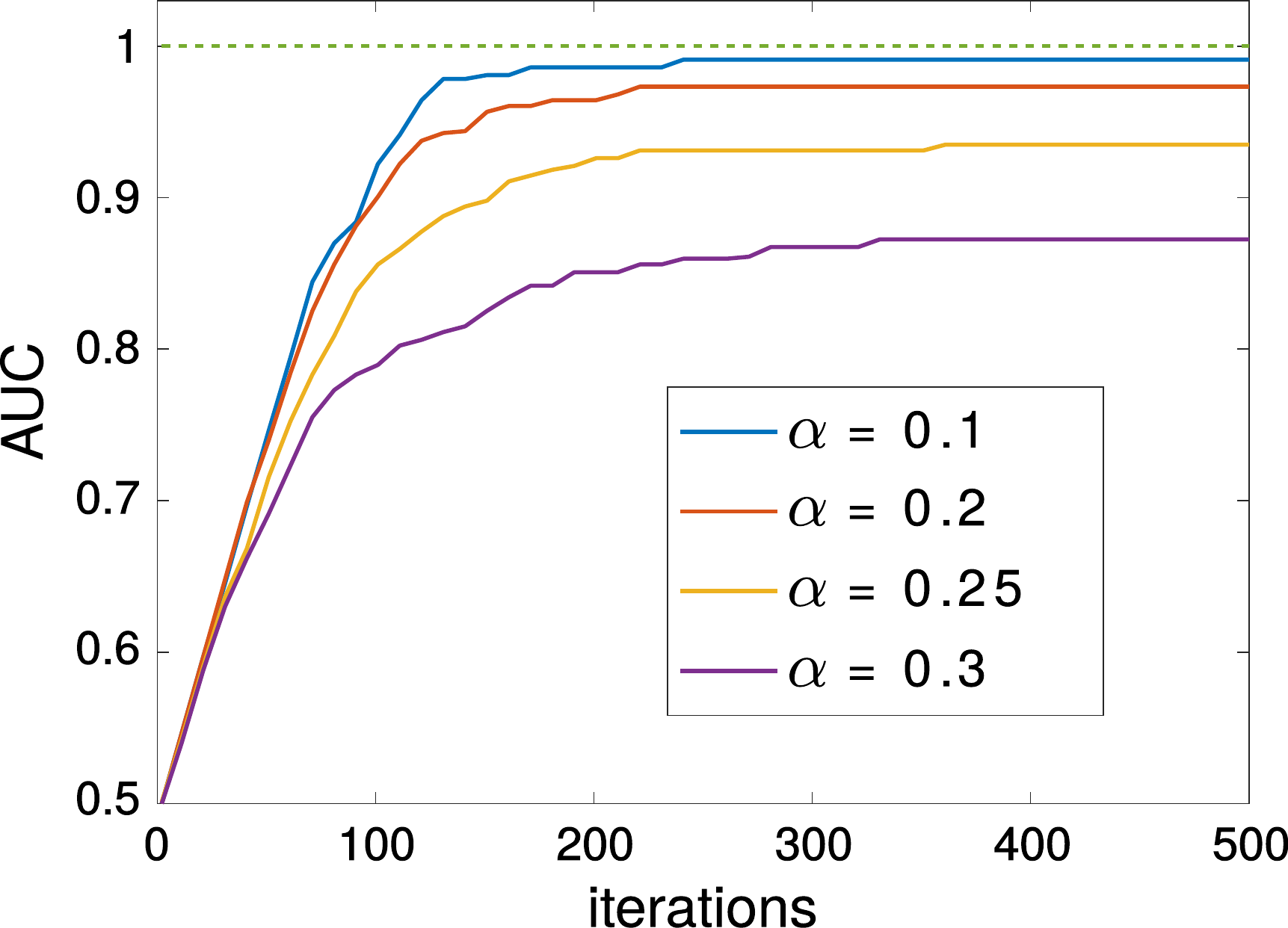}
\label{fAUC_entry}}
 \caption{Similarity recovery experiment on synthetic data. Figure~\ref{fAUC} and Figure~\ref{fAUC_entry} show the AUC scores (for feature recovery and entry recovery respectively) along the iterations of the algorithm for different values of $\alpha$.}\label{fMetricRecovery_AUC}
\end{figure}

\begin{figure}[t]
\centering
\subfigure[True similarity]{\includegraphics[width=0.45\textwidth]{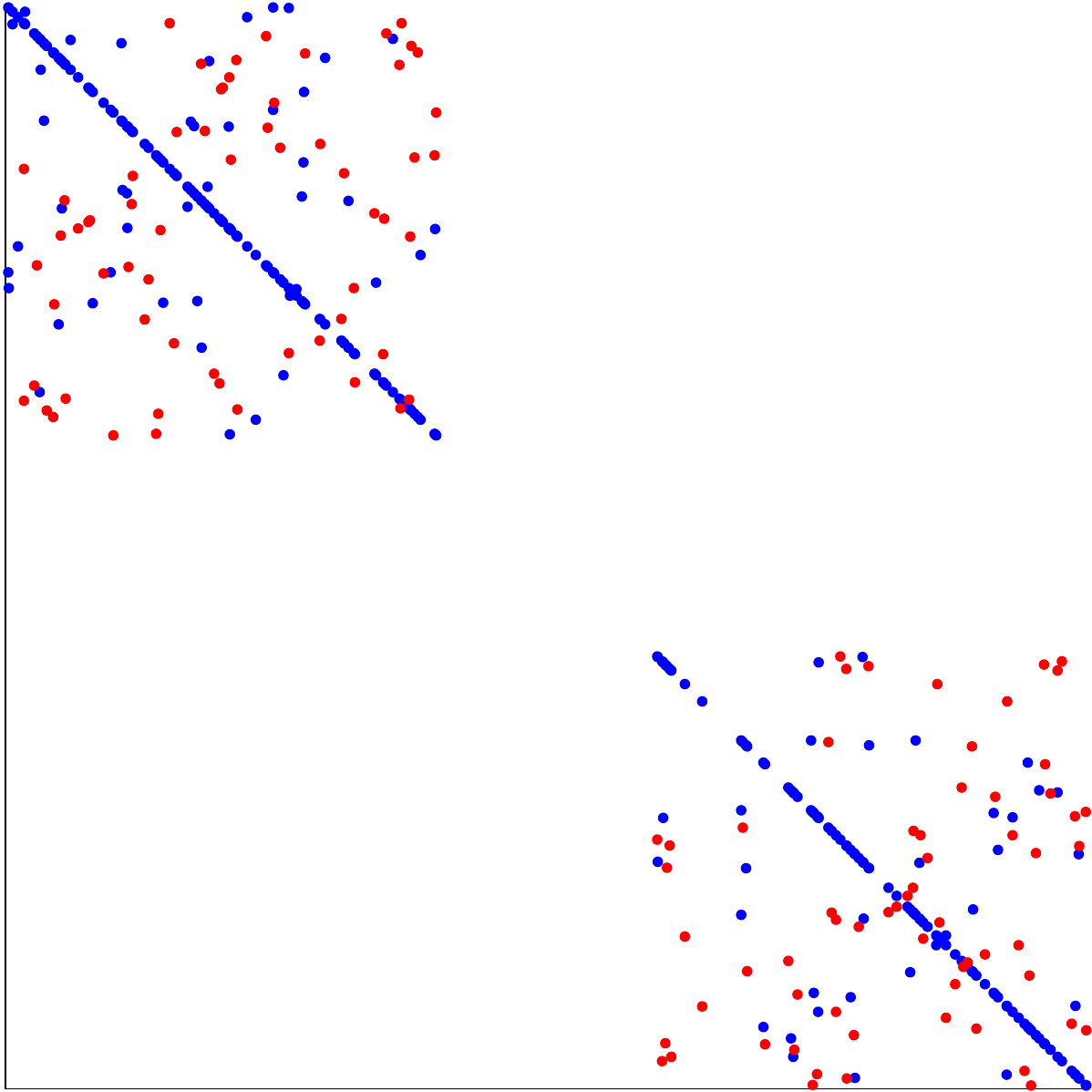}
\label{fTrueMetric}}
\subfigure[Learned similarity ($\alpha=20$\%)]{\includegraphics[width=0.45\textwidth]{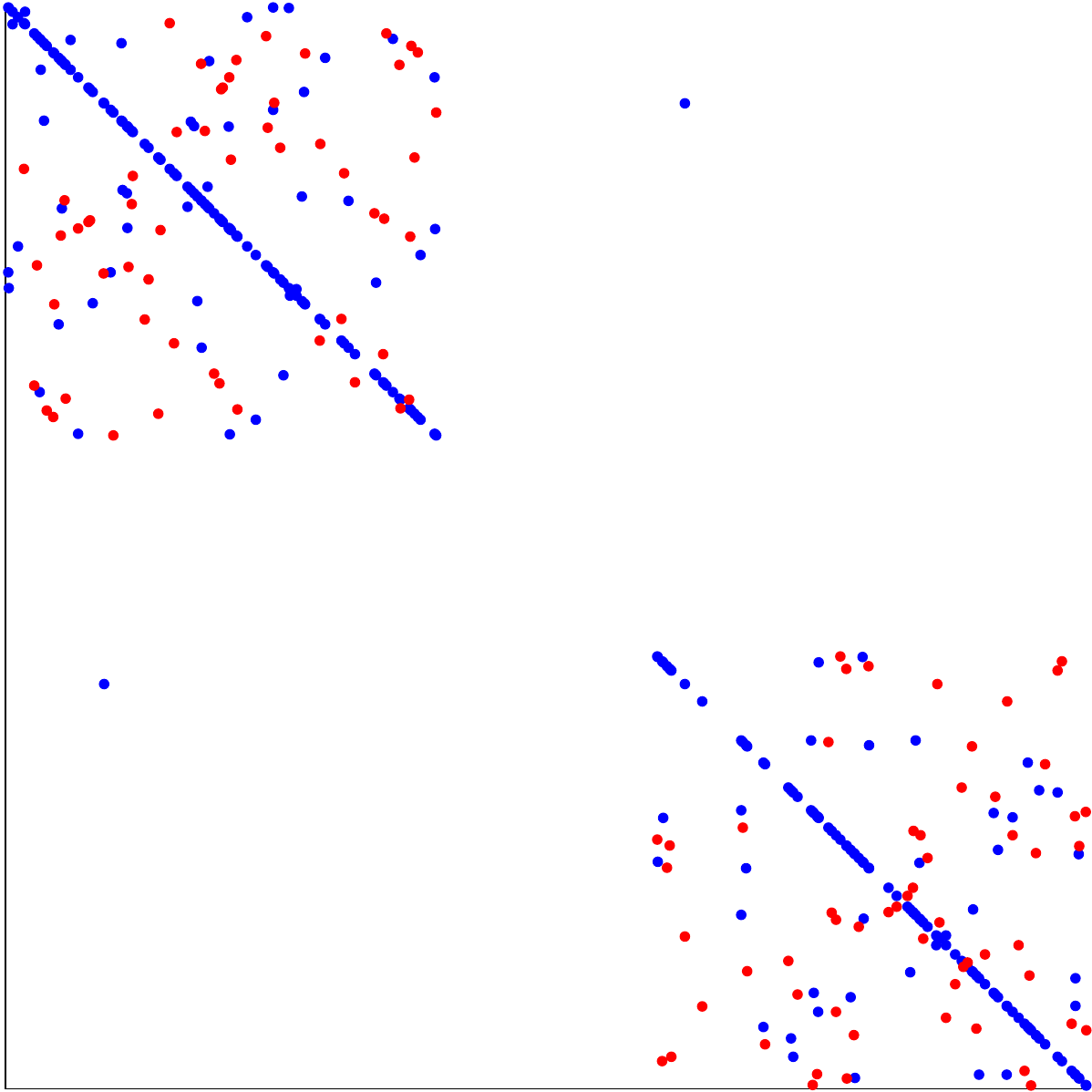}
\label{fLearned}}
 \caption{Similarity recovery experiment on synthetic data. Figure~\ref{fTrueMetric} shows the underlying ground truth similarity, where blue dots represent positive entries and red dots represent negative entries (combination coefficients are not displayed). Figure~\ref{fLearned} shows the similarity learned by \hdsl ($\alpha=$20\%), which is visually very close to the ground truth.}\label{fMetricRecovery_matrix}
\end{figure}

\paragraph{Results} 
We aim to measure how accurately we recover the entries (i.e., pairs of features) that are active in the ground truth similarity as training proceeds. To do so, at each iteration $k$ of \hdsl, we rank each pair of features by descending order of the absolute value of the corresponding entry in the current matrix $\mat{M}^{(k)}$. We then compute the Area under the ROC Curve (AUC) of the ranking induced by the similarity with respect to the list of active entries in the ground truth similarity. The AUC is well-suited to the imbalanced setting (as active entries in the ground truth are a small subset of all entries). It can be interpreted as the probability that a random entry that is active in the ground truth is ranked higher than a random inactive one. Following a similar process, we also compute an AUC score for individual features: this is done by ranking each feature by the $L_1$ norm of its associated row in the matrix.

The AUC scores for feature and entry recovery along the iterations are reported in Figure~\ref{fMetricRecovery_AUC} for different values of $\alpha$. When the quality of the triplet constraints is high ($\alpha=$10\%,20\%), the AUC increases quickly to converge very close to $1.0$, indicating an almost perfect recovery of relevant features/entries. This confirms that \hdsl is able to accurately identify the small number of correct features and pairs of features. As $\alpha$ increases (i.e., the similarity constraints become noisy and less informative), the AUC increases at a slower pace and the final value decreases. This is expected as the quality of the information carried by the similarity judgments is key to recover the ground truth similarity. Yet, even for $\alpha=$30\%, the final AUC score is still very high (above $0.85$ for both feature and entry recovery).
This good recovery behavior is confirmed by the visual representations of the ground truth and learned similarity matrices shown in Figure~\ref{fMetricRecovery_matrix}. We observe that the learned similarity (when $\alpha=20$\%) clearly recovers the block structure of the true similarity, and is able to correctly identify most individual entries with very few false positives.

\subsubsection{Link Prediction}

We now investigate the ability of our algorithm to generalize well as the feature dimensionality increases by conducting a signed link prediction experiment, which is the task of distinguishing positive and negative interactions in a network \citep[see e.g.][]{agrawal2013link}.

We generate $500$ samples with different number of features $d$ ranging from $5,000$ to $1,000,000$. As the dimension $d$ increases, we decrease the average sparsity of data (from $0.02$ to $0.002$) to limit running time.
In real high-dimensional datasets, features typically do not appear in a uniform frequency: instead, a small portion of features tends to dominate the others.  Following this observation, we generate features whose frequency follow a power law style distribution, as shown in Figure~\ref{f_dimfreq}.
The ground truth similarity is then a convex combination of randomly selected bases as in the previous experiment, except that we restrict the selected bases to those involving features that are frequent enough (a frequency of at least 0.1 was chosen for this experiment). This is needed to ensure that the features involved in the ground truth similarity will occur at least a few times in our small dataset, but we emphasize that the algorithm is exposed to the entire feature set and does not know which features are relevant. 

Based on the samples and the ground truth similarity, we generate signed link observations of the form $\{x^i_1, x^i_2, y^i\}_i^N$ ($y^i\in\{-1, 1\}$). We associate the label $y^i=1$ (positive link) to pairs for which the similarity between $x_1$ and $x_2$ ranks in the top 5\% of $x_1$'s (or $x_2$'s) neighbors according to the ground truth similarity measure. On the other hand, $y^i=-1$ (negative link) indicates that the similarity ranks in the bottom 5\% of $x_1$'s (or $x_2$'s) neighbors. We split these link observations into training, validation and test sets of $1,000$ observations each. Triplets constraints are generated from training links --- given a pair $x_1, x_2, y$, we randomly sample $x_3$ as a similar (if $y=-1$) or dissimilar (if $y = 1$) node. The validation set is used to tune the hyperparameter $\lambda$ and for early stopping.




\begin{figure}[t]
\centering
\subfigure[Feature frequency ($d=50,000$)]{\includegraphics[width=0.43\textwidth]{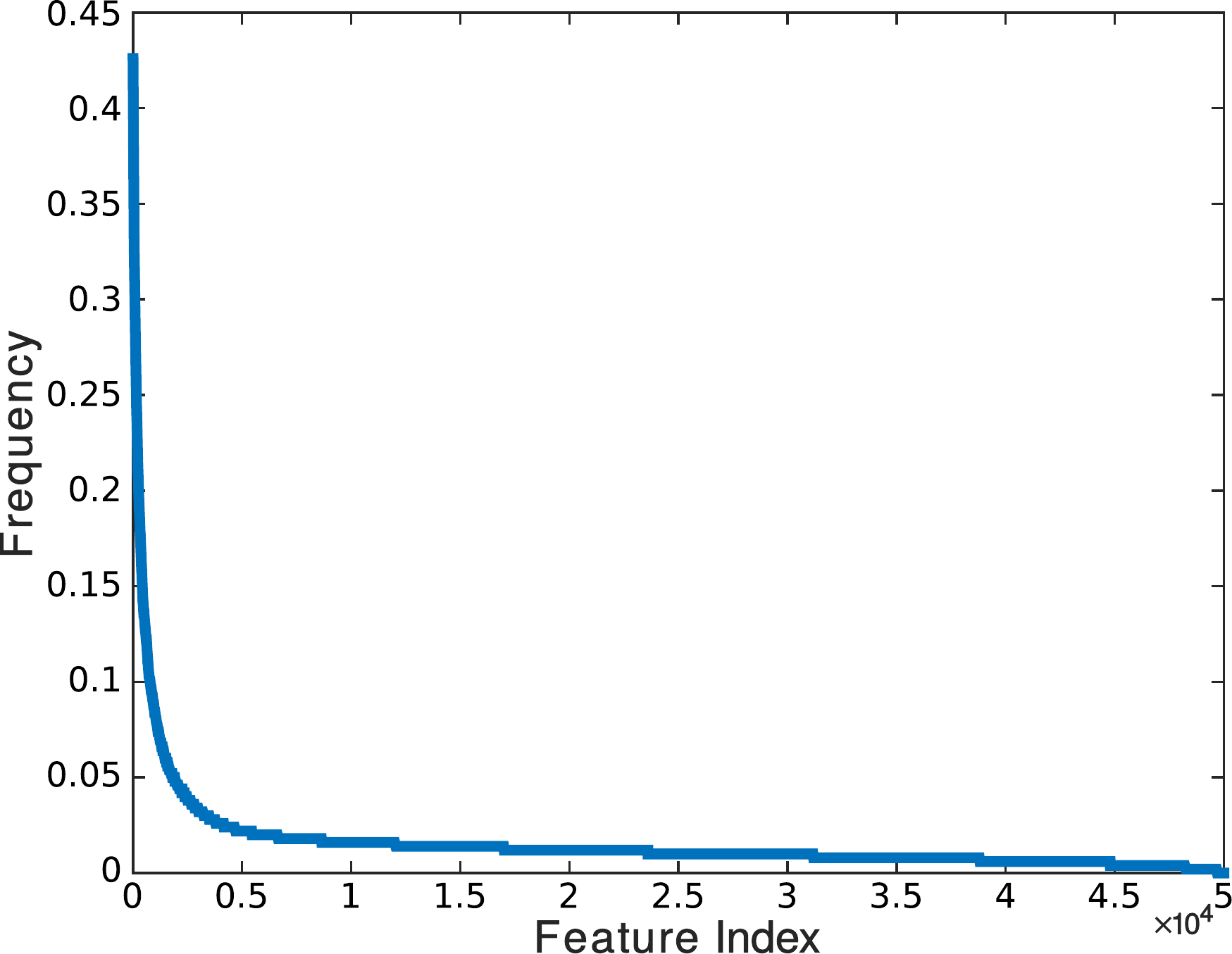}\label{f_dimfreq}}\hspace*{.5cm}
\subfigure[AUC scores on the test set]{\includegraphics[width=0.52\textwidth]{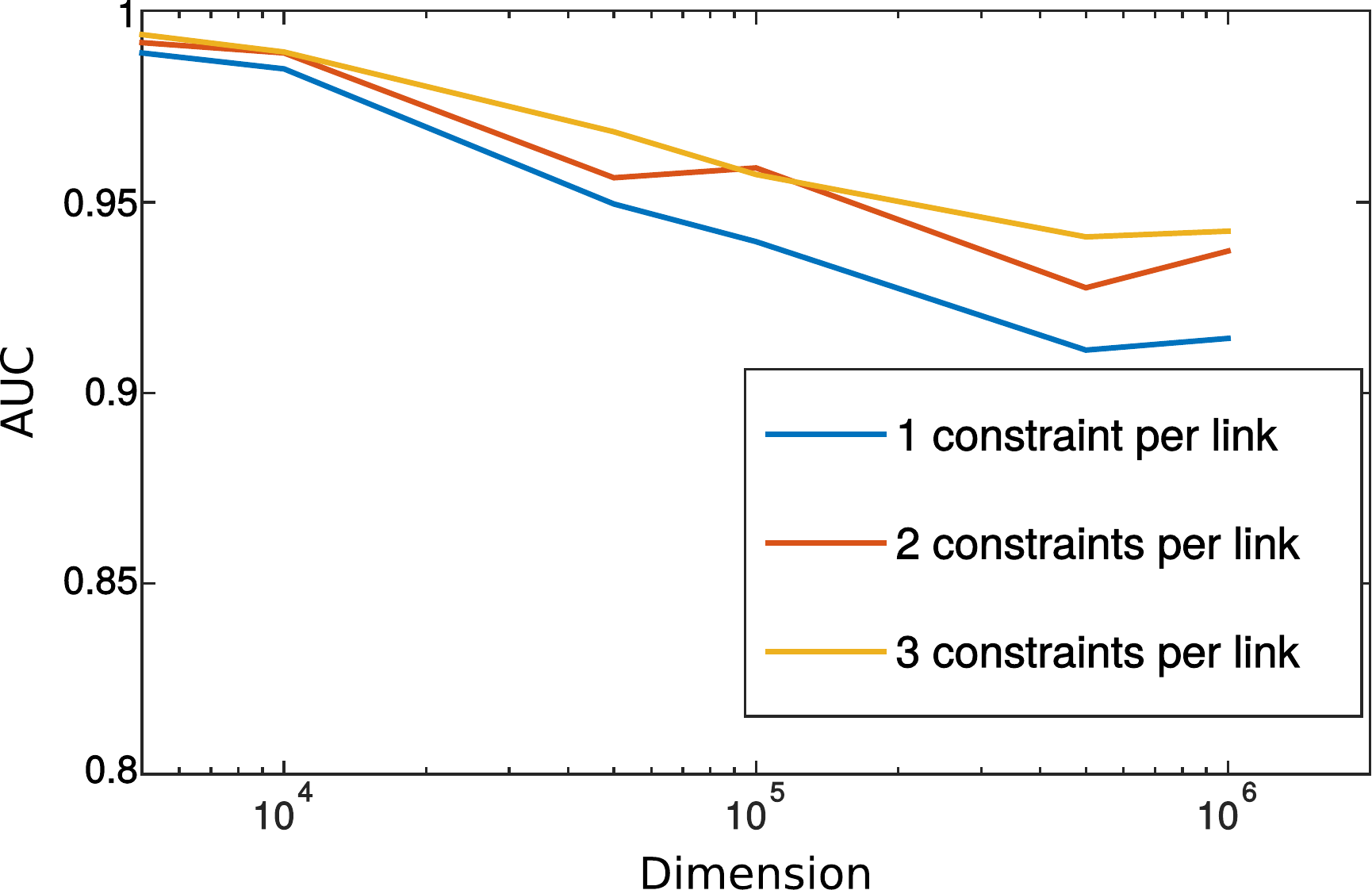}\label{fAUC2}}
\caption{Link prediction experiment on synthetic data. Figure~\ref{f_dimfreq} shows the feature frequency distribution, which follows a power law as in many real high-dimensional datasets. Figure~\ref{fAUC2} shows AUC scores on the test set for different number of features (in log scale) and number of training constraints per link.}\label{fLink}
\end{figure}

\paragraph{Results}
We measure the generalization ability of \hdsl by the AUC score of link prediction on the test set. Figure~\ref{fAUC2} reports these AUC scores across different dimensions. We also show results for different numbers of constraints per training link. The results are averaged over 5 random runs.
As one would expect, the task becomes increasingly difficult as the dimension becomes larger, since the size of the training set is fixed ($1,000$ training links generated from $500$ nodes). However, the performance decreases slowly (roughly logarithmically) with the dimension, and we achieve very high AUC scores (larger than 0.9) even for one million features. We also see that training from more constraints tends to improve the prediction performance. 


\subsection{Experiments on Real Datasets}
\label{sec:real}

We now present comparative experiments on several high-dimensional real datasets, evaluating our approach against several baselines and competing methods.

\subsubsection{Setup}

\paragraph{Datasets} We report experimental results on several real-world classification datasets with up to 100,000 features. Dorothea and dexter come from the NIPS 2003 feature selection challenge \citep{Guyon2004} and are respectively pharmaceutical and text data with predefined splitting into training, validation and test sets. They both contain a large proportion of noisy/irrelevant features. Reuters CV1 is a popular text classification dataset with bag-of-words representation. We use the binary classification version from the LIBSVM dataset collection\footnote{\url{http://www.csie.ntu.edu.tw/~cjlin/libsvmtools/datasets/}} (with 60\%/20\%/20\% random splits) and the 4-classes version (with 40\%/30\%/30\% random splits) introduced by \citet{Cai2012}. Detailed information on the datasets and splits is given in Table \ref{tDatasets}. All datasets are normalized such that each feature takes values in $[0,1]$.

\begin{table}[!t]
\centering
\small
\begin{tabular}{|c||c|c|c|c|c|}
\hline
Datasets     & Dimension & Sparsity & Training size & Validation size & Test size \\ \hline\hline
dexter       & 20,000       & 0.48\%   & 300           & 300        & 2,000      \\ \hline
dorothea     & 100,000      & 0.91\%   & 800           & 350        & 800       \\ \hline
rcv1\_2 & 47,236    & 0.16\%   & 12,145         & 4,048       & 4,049      \\ \hline
rcv1\_4      & 29,992    & 0.26\%   & 3,850          & 2,888       & 2,887      \\ \hline
\end{tabular}
\caption{Datasets used in the experiments}
\label{tDatasets}
\end{table}


\paragraph{Competing methods} We compare the proposed approach (\textsc{hdsl}) to several methods:
\begin{itemize}
\item \ident: The standard dot product, which is equivalent to setting $\mat{M} = \mat{I}$. 
\item \mldiag: Diagonal similarity learning (i.e., a weighting of the features), as done in \cite{Gao2014}. We obtain it by minimizing the same loss as in \hdsl with $\ell_2$ and $\ell_1$ regularization, i.e.,
\begin{equation*}
\min_{\vct{w}\in\mathbb{R}^d} \quad f(\vct{w}) = \frac{1}{T}\sum_{t=1}^T \ell\left(\innerp{\mat{A}^t,\diag(\vct{w})}\right)+\lambda\Omega(\vct{w}),
\end{equation*}
where $\Omega(\vct{w}) \in \{\|\vct{w}\|_2^2,\|\vct{w}\|_1\}$ and $\lambda$ is the regularization parameter. Optimization was done using (proximal) gradient descent.
\item \rpml: Similarity learning in random projected space. Given $r\ll d$, let $\mat{R}\in\mathbb{R}^{d\times r}$ be a matrix where each entry $r_{ij}$ is randomly drawn from $\mathcal{N}(0,1)$. For each data instance $\vct{x}\in\mathbb{R}^d$, we generate $\tilde{\vct{x}} =\frac{1}{\sqrt{r}}\mat{R}\T\vct{x} \in \mathbb{R}^r$ and use this reduced data in OASIS \citep{Chechik2009}, a fast online method to learn a bilinear similarity from triplet constraints.
\item \pcaml: Similarity learning in PCA space. Same as \rpml, except that PCA is used instead of random projections to project the data into $\mathbb{R}^r$.
\item \svm: Support Vector Machines. We use linear SVM, which is known to perform well for sparse high-dimensional data \citep{Caruana2008}, with $\ell_2$ and $\ell_1$ regularization. We also use nonlinear SVM with the polynomial kernel (2nd and 3rd degree) popular in text classification \citep{Chang2010}. The SVM models are trained using liblinear \citep{Fan2008} and libsvm \citep{Chang2011} with 1vs1 paradigm for multiclass.
\end{itemize}

We have also tried to compare our method with \textsc{Comet} \citep{Atzmon2015}, which also learns a bilinear similarity in a greedy fashion with rank-1 updates. However, as mentioned in Section~\ref{sec:related_metric} their coordinate descent algorithm has a time complexity of $O(d^2)$ per iteration, as well as overall memory complexity of $O(d^2)$. We run the sparse version of code provided by the authors\footnote{\url{https://github.com/yuvalatzmon/COMET}} on a machine with a 2.3GHz Intel Core i7 and 16GB memory. On the dexter dataset (which has the smallest dimensionality in our benchmark), a single pass over the features took more than 4 hours, while the authors reported that about 10 passes are generally needed for \textsc{Comet} to converge \citep{Atzmon2015}. On the dorothea dataset, \textsc{Comet} returned a memory error. As a result, we did not include \textsc{Comet} to our empirical comparison. In contrast, on the same hardware, our approach \hdsl takes less than 1 minute on dexter and less than 1 hour on dorothea.


\paragraph{Training Procedure} For all similarity learning algorithms, we generate 15 training constraints for each instance by identifying its 3 target neighbors (nearest neighbors with same label) and 5 impostors (nearest neighbors with different label), following \cite{Weinberger2009}. Due to the very small number of training instances in dexter, we found that better performance is achieved across all methods when using 20 training constraints per instance, drawn at random based on its label.
All parameters are tuned using the accuracy on the validation set. For \hdsl, we use the fast heuristic described in Section~\ref{sec:approx} and tune the scale parameter $\lambda \in \{1,10,\dots,10^9\}$. The regularization parameters of \mldiag and \svm are tuned in $\{10^{-9},\dots,10^8\}$ and the ``aggressiveness'' parameter of OASIS is tuned in $\{10^{-9},\dots,10^2\}$.

\subsubsection{Results} 

\begin{table}[t]
\centering
\small
\begin{tabular}{|c||c|c|c|c|c|c|}
\hline
Dataset     & \ident  & \rpml & \pcaml &\mldiag-$\ell_2$ & \mldiag-$\ell_1$ & \hdsl \\ \hline\hline
dexter       & 20.1 & 24.0  {[}1000{]} & 9.3 [50] &8.4       & 8.4  {[}773{]}              & \textbf{6.5}  {[}183{]}         \\ \hline
dorothea     & 9.3    & 11.4 {[}150{]} & 9.9 [800]  & 6.8       & 6.6  {[}860{]}              & \textbf{6.5} {[}731{]}          \\ \hline
rcv1\_2 & 6.9  & 7.0  {[}2000{]}  & 4.5 [1500] &3.5       & 3.7  {[}5289{]}             & \textbf{3.4}  {[}2126{]}        \\ \hline
rcv1\_4      & 11.2 & 10.6  {[}1000{]} & 6.1 	[800] &6.2       & 7.2  {[}3878{]}             & \textbf{5.7}  {[}1888{]}        \\ \hline
\end{tabular}
\caption{$k$-NN test error (\%) of the similarities learned with each method. The number of features used by each similarity (when smaller than $d$) is given in brackets. Best accuracy on each dataset is shown in bold.}
\label{tKnnErr}
\end{table}

\begin{table}[t]
\centering
\small
\begin{tabular}{|c||c|c|c|c|c|}
\hline
Dataset     & \svm-poly-2 & \svm-poly-3 & \svm-linear-$\ell_2$ & \svm-linear-$\ell_1$ & \hdsl \\ \hline\hline
dexter       & 9.4                   & 9.2                   & 8.9        & 8.9  {[}281{]}                     & \textbf{6.5}  {[}183{]}         \\ \hline
dorothea     & 7                     & 6.6                   & 8.1        & 6.6  {[}366{]}                    & \textbf{6.5} {[}731{]}          \\ \hline
	rcv1\_2 & 3.4                   & \textbf{3.3}                   & 3.5        & 4.0  {[}1915{]}                    & 3.4  {[}2126{]}        \\ \hline
rcv1\_4      & 5.7                   & 5.7                   & \textbf{5.1}        & 5.7  {[}2770{]}                    & 5.7 {[}1888{]}        \\ \hline
\end{tabular}
\caption{Test error (\%) of several SVM variants compared to \hdsl. As in Table~\ref{tKnnErr}, the number of features is given in brackets and best accuracies are shown in bold. 
}
\label{tCompareSVMs}
\end{table}

\paragraph{Classification Performance} We first investigate the performance of each similarity learning approach in $k$-NN classification ($k$ was set to 3 for all experiments). For \rpml and \pcaml, we choose the dimension $r$ of the reduced space based on the accuracy of the learned similarity on the validation set, limiting our search to $r\leq 2000$ because OASIS is extremely slow beyond this point.\footnote{Note that the number of PCA dimensions is at most the number of training examples. Therefore, for dexter and dorothea, $r$ is at most 300 and 800 respectively.} Similarly, we use the performance on validation data to do early stopping in \hdsl, which also has the effect of restricting the number of features used by the learned similarity.

Table~\ref{tKnnErr} shows the $k$-NN classification performance. We can first observe that \rpml often performs worse than \ident, which is consistent with previous observations showing that a large number of random projections may be needed to obtain good performance \citep{Fradkin2003}. \pcaml gives much better results, but is generally outperformed by a simple diagonal similarity learned directly in the original high-dimensional space. \hdsl, however, outperforms all other algorithms on these datasets, including \mldiag. This shows the good generalization performance of the proposed approach, even though the number of training samples is sometimes very small compared to the number of features, as in dexter and dorothea. It also shows the relevance of encoding ``second order'' information (pairwise interactions between the original features) in the similarity instead of considering a simple weighting of features as in \mldiag.


Table~\ref{tCompareSVMs} shows the comparison with SVMs. Interestingly, \hdsl outperforms all SVM variants on dexter and dorothea, both of which have a large proportion of irrelevant features. This shows that its greedy strategy and early stopping mechanism achieves better feature selection and generalization than the $\ell_1$ version of linear SVM. On the other two datasets, \hdsl is competitive with SVM, although it is outperformed slightly by one variant (\svm-poly-3 on rcv1\_2 and \svm-linear-$\ell_2$ on rcv1\_4), both of which rely on all features.

\paragraph{Feature selection and sparsity} 

\begin{figure}[t]
\centering
\subfigure[dexter dataset]{\includegraphics[width=0.45\textwidth]{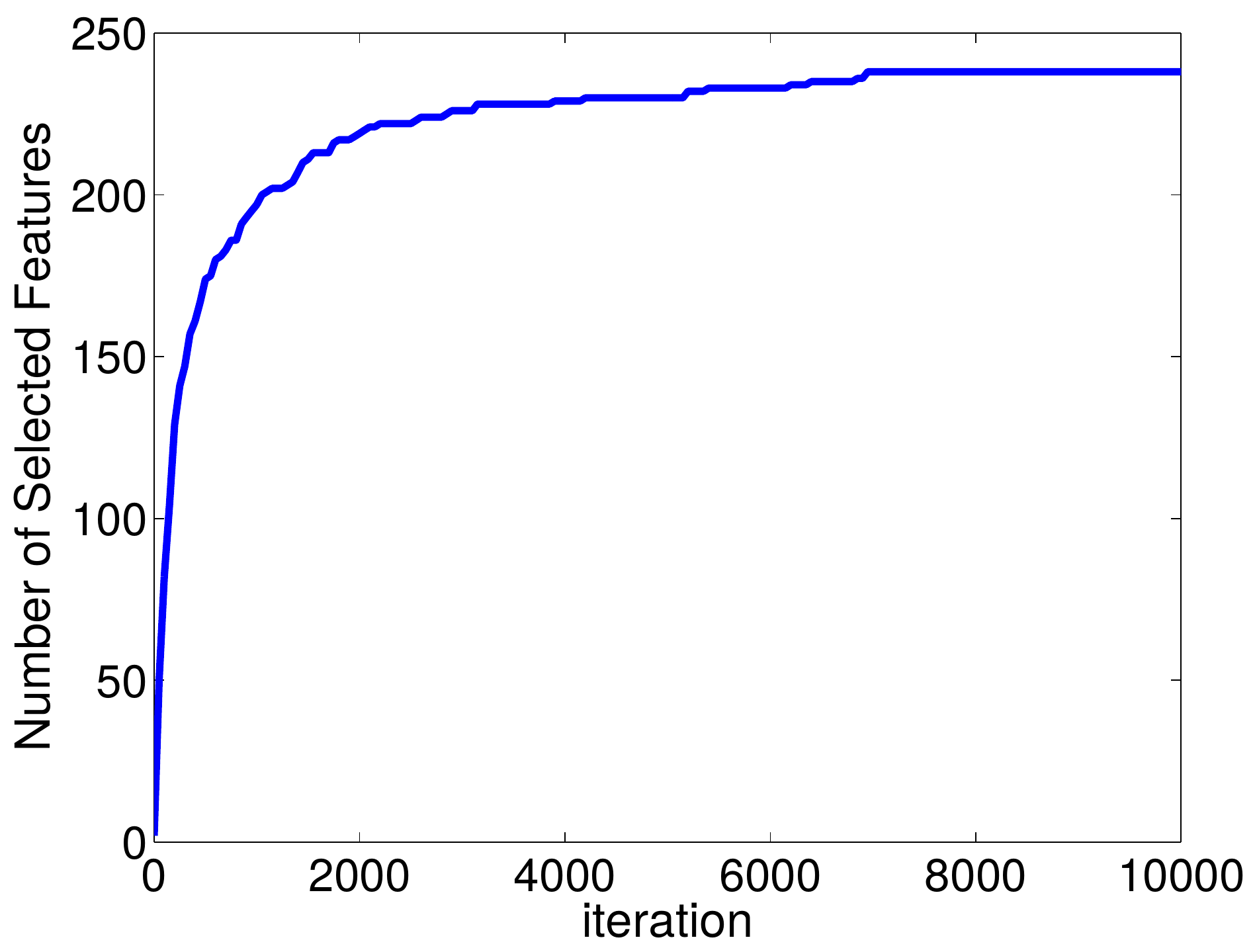}
\label{fNumFea1}}
\subfigure[rcv1\_4 dataset]{\includegraphics[width=0.45\textwidth]{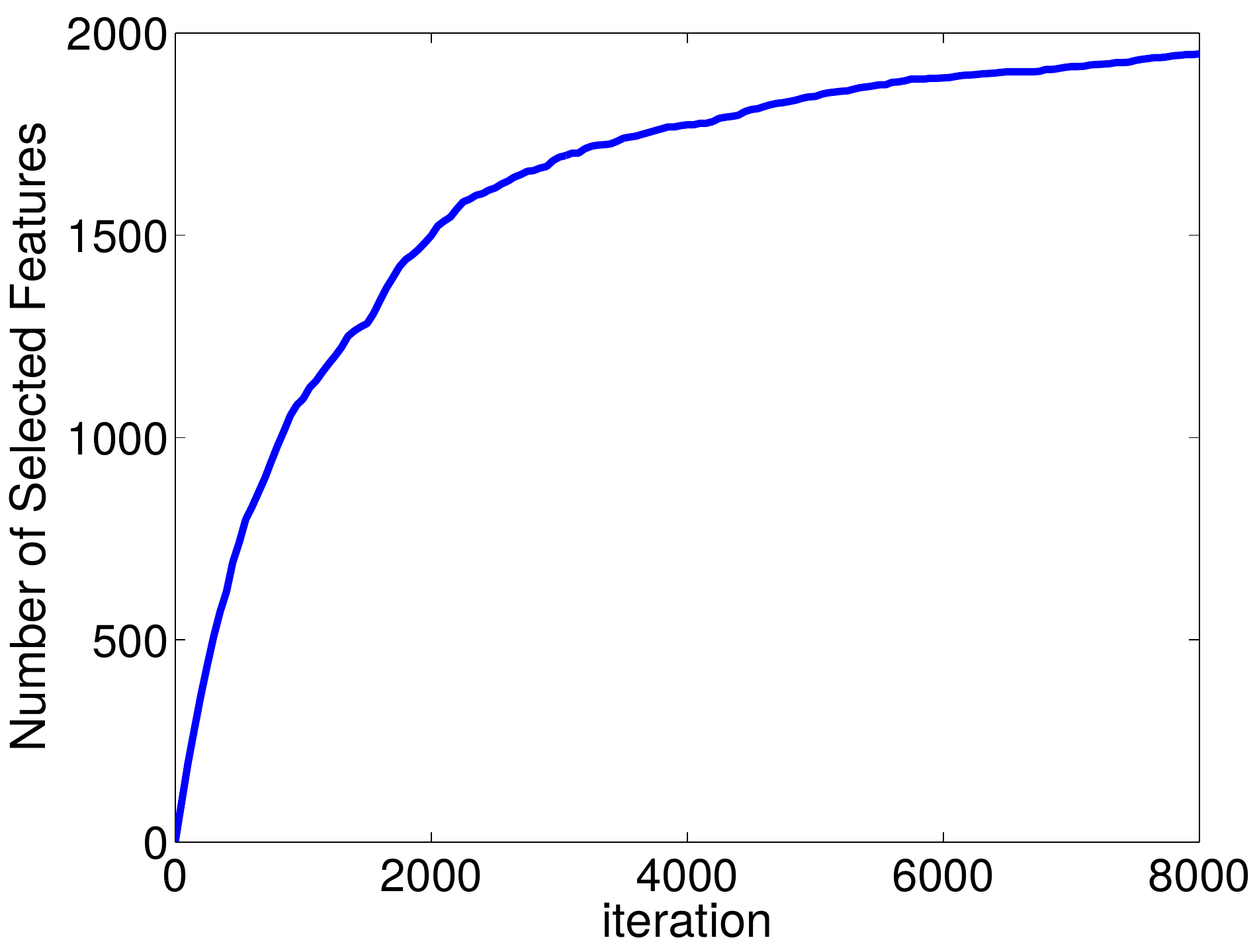} 
\label{fNumFea2}}
 \caption{Number of active features learned by \hdsl as a function of the iteration number.}\label{fSparseFeature}
 \label{fFeaNum}
\end{figure}

\begin{figure}[t]
\centering
\subfigure[dexter ($20,000\times20,000$ matrix, 712 nonzeros)]{\includegraphics[width=0.43\textwidth]{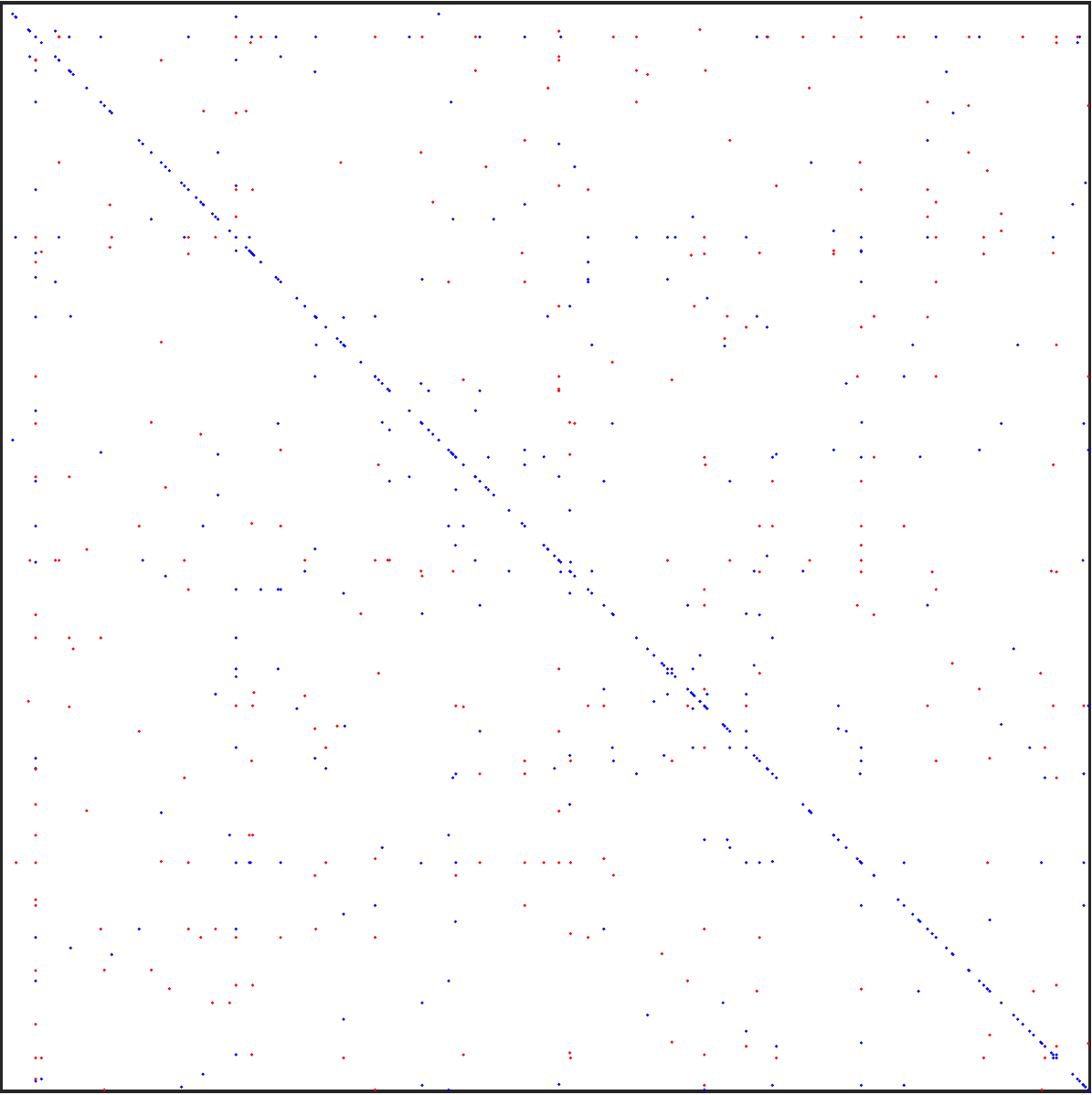}
\label{dexter_mat}}
\hspace*{1cm}\subfigure[rcv1\_4 ($29,992\times29,992$ matrix, 5263 nonzeros)]{\includegraphics[width=0.43\textwidth]{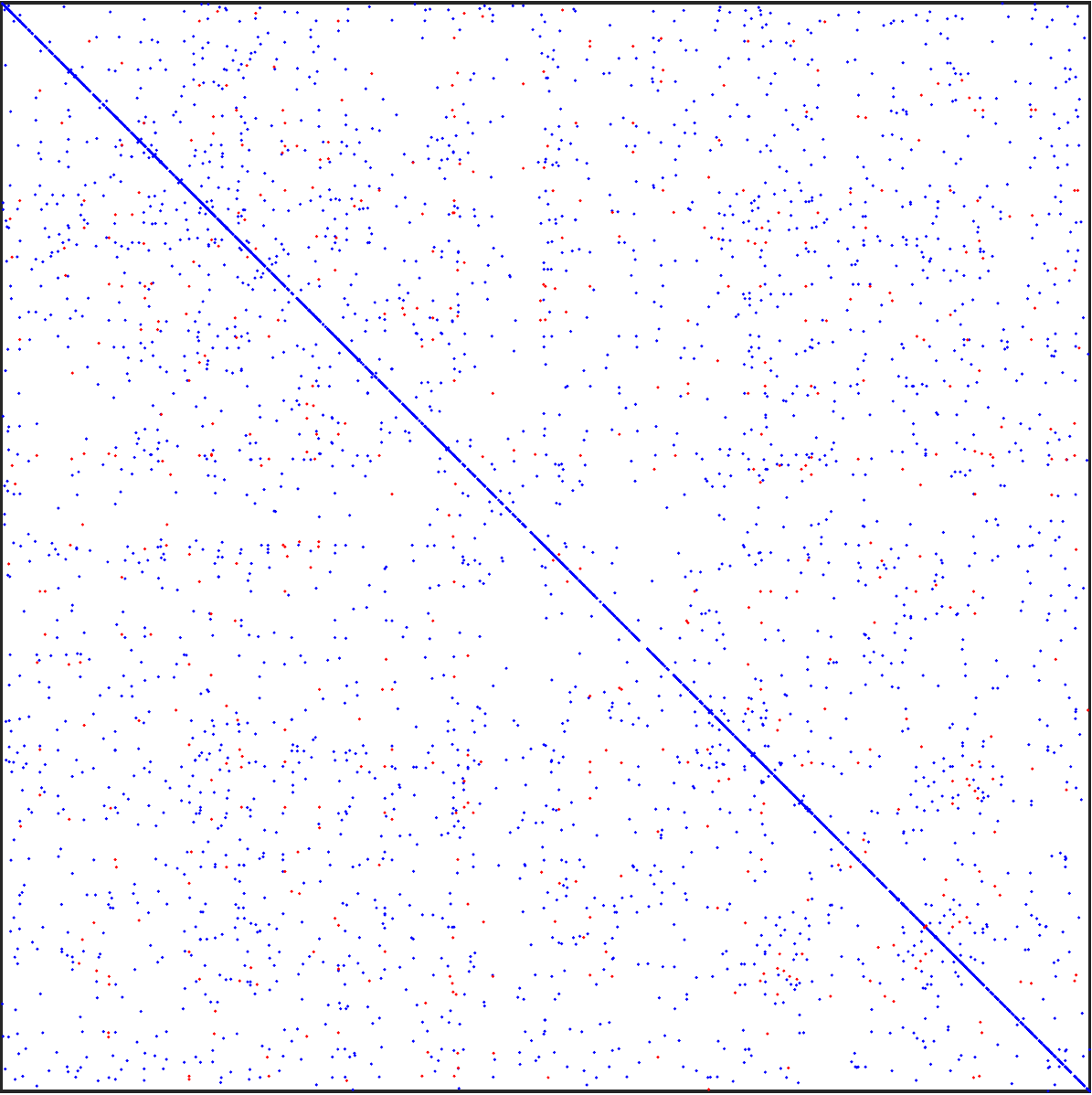} 
\label{rcv14_mat}}
 \caption{Sparsity structure of the matrix $\mat{M}$ learned by \hdsl. Positive and negative entries are shown in blue and red, respectively (best seen in color).}\label{}
 \label{hdsl_mat}
\end{figure}

We now focus on the ability of \hdsl to perform feature selection and more generally to learn sparse similarity functions. To better understand the behavior of \hdsl, we show in Figure~\ref{fFeaNum} the number of selected features as a function of the iteration number for two of the datasets. Remember that at most two new features can be added at each iteration. Figure~\ref{fFeaNum} shows that \hdsl incorporates many features early on but tends to eventually converge to a modest fraction of features (the same observation holds for the other two datasets). This may explain why \hdsl does not suffer much from overfitting even when training data is scarce as in dexter.

Another attractive characteristic of \hdsl is its ability to learn a matrix that is sparse not only on the diagonal but also off-diagonal (the proportion of nonzero entries is in the order of 0.0001\% for all datasets). In other words, the learned similarity only relies on a few relevant pairwise interactions between features. Figure~\ref{hdsl_mat} shows two examples, where we can see that \hdsl is able to exploit the product of two features as either a positive or negative contribution to the similarity score. This opens the door to an analysis of the importance of pairs of features (for instance, word co-occurrence) for the application at hand. Finally, the extreme sparsity of the matrices allows very fast similarity computation. Together with the superior accuracy brought by \hdsl, it makes our approach potentially useful in a variety of contexts ($k$-NN, clustering, ranking, etc).

Finally, it is also worth noticing that \hdsl uses significantly less features than \mldiag-$\ell_1$ (see numbers in brackets in Table~\ref{tKnnErr}). We attribute this to the extra modeling capability brought by the non-diagonal similarity observed in Figure~\ref{hdsl_mat}.\footnote{Note that \hdsl uses roughly the same number of features as \svm-linear-$\ell_1$ (Table~\ref{tCompareSVMs}), but it is difficult to draw any solid conclusion because the objective and training data for each method are different, and SVM is a combination of binary models.}





\paragraph{Dimension reduction}

\begin{figure}[t]
\centering
\subfigure[dexter dataset]{
\includegraphics[width=0.48\textwidth]{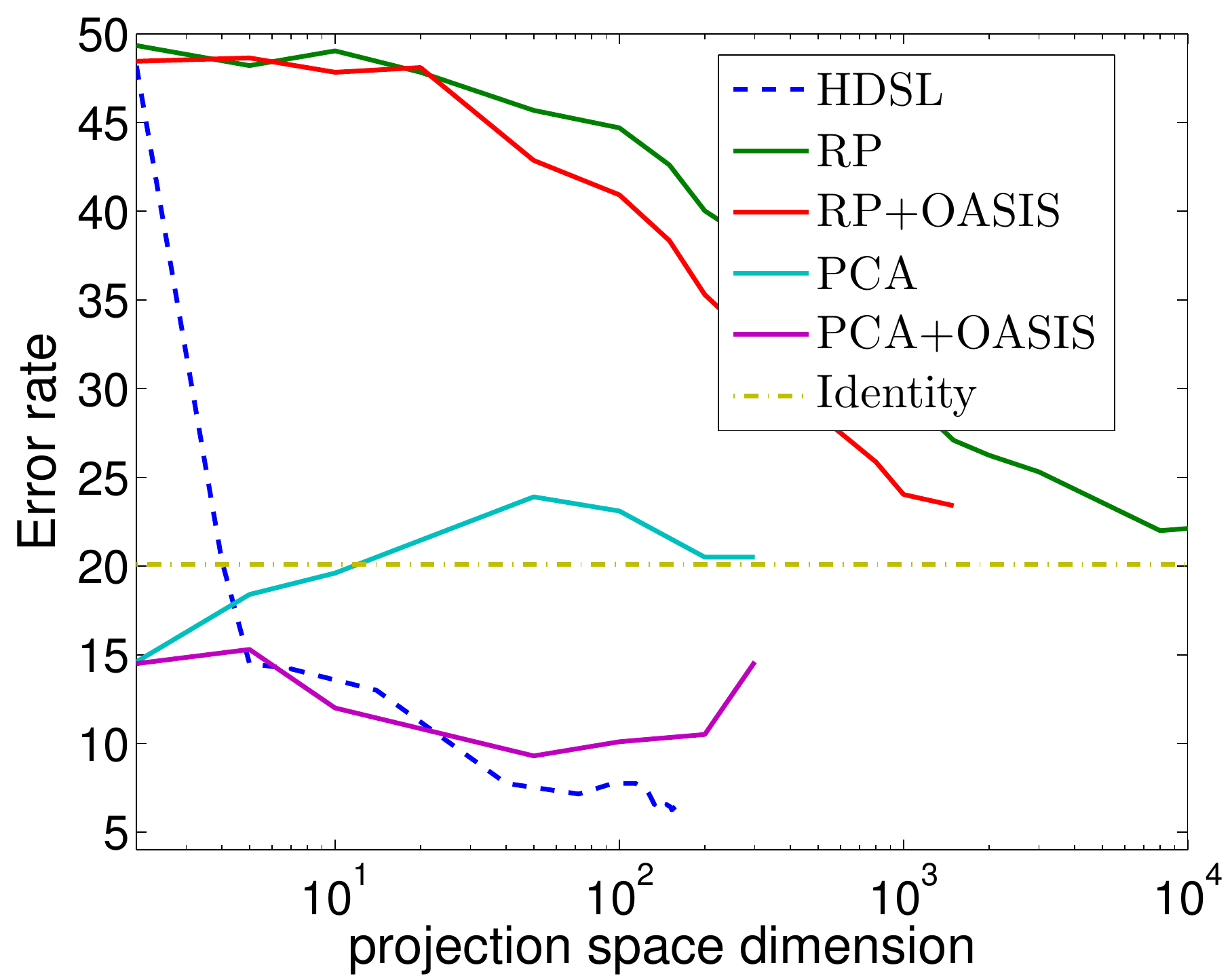}
\label{fErr_dim1}
}
\subfigure[dorothea dataset]{\includegraphics[width=0.48\textwidth]{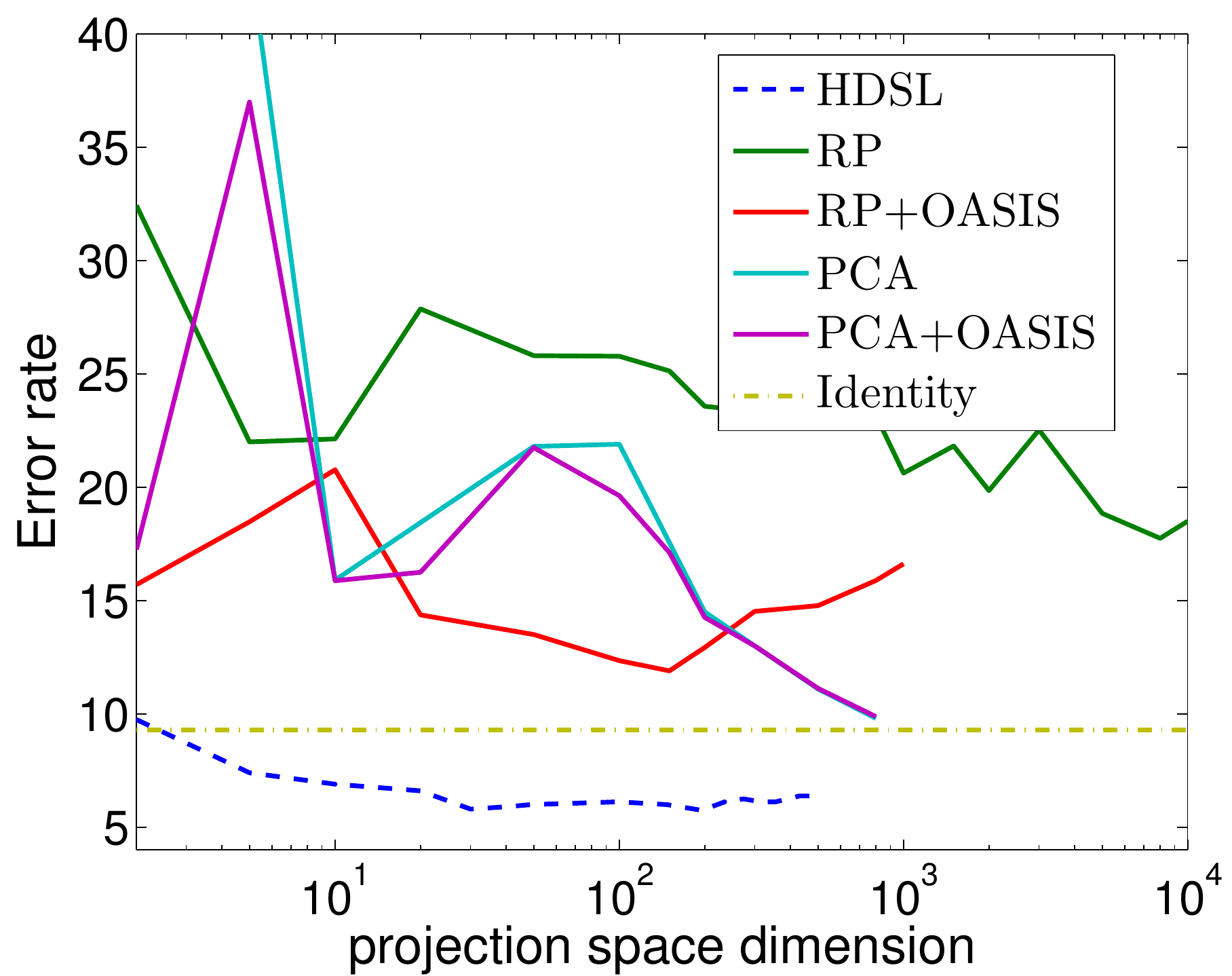} 
\label{fErr_dim4}}
\subfigure[rcv1\_2 dataset]{\includegraphics[width=0.48\textwidth]{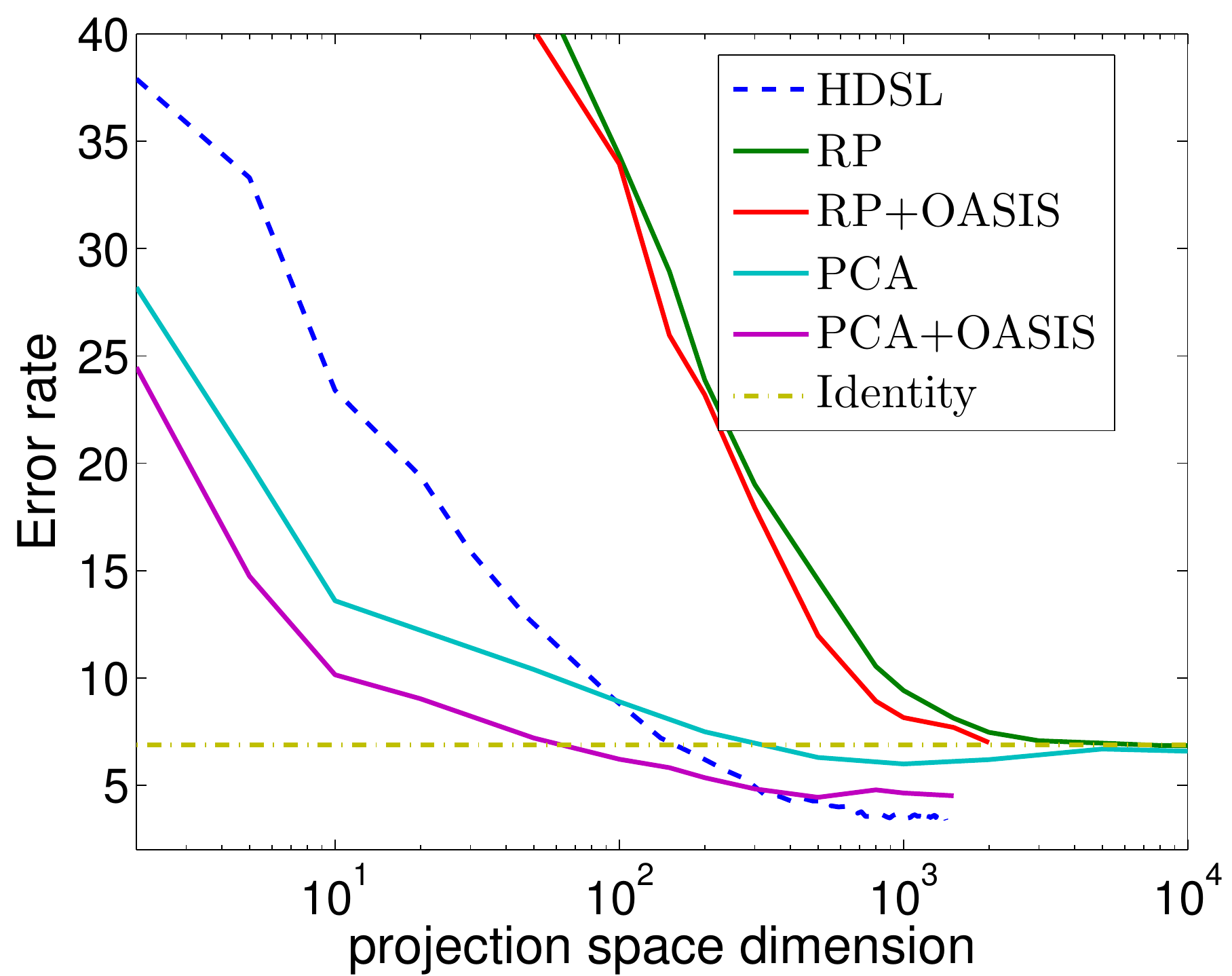} 
\label{fErr_dim2}}
\subfigure[rcv1\_4 dataset]{
\includegraphics[width=0.48\textwidth]{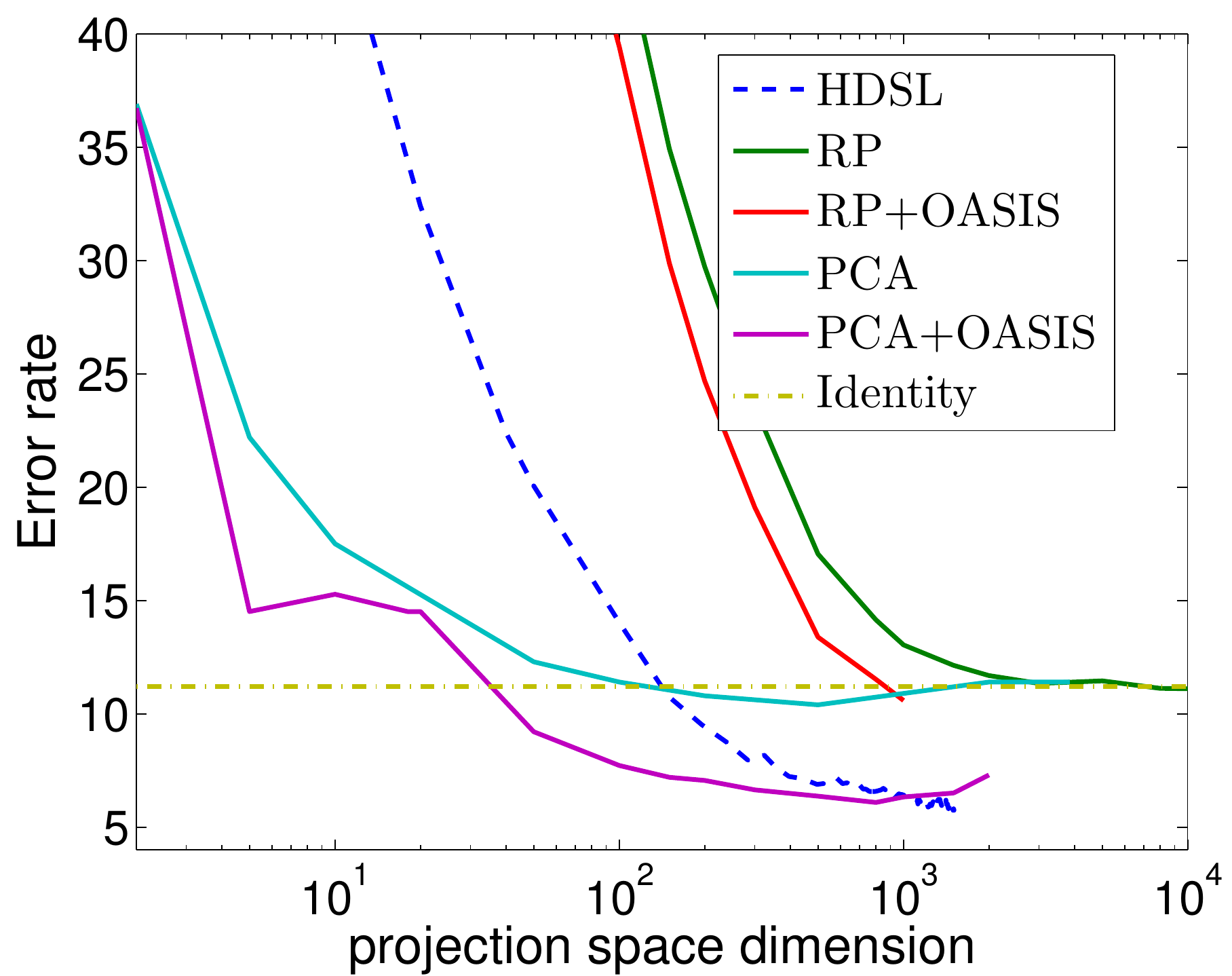}
\label{fErr_dim3}
}
 \caption{$k$-NN test error as a function of the dimensionality of the space (in log scale). Best seen in color.}\label{fErr_dim}
\end{figure}

We now investigate the potential of \hdsl for dimensionality reduction. Recall that \hdsl learns a sequence of PSD matrices $\mat{M}^{(k)}$. We can use the square root of $\mat{M}^{(k)}$ to project the data into a new space where the dot product is equivalent to $S_{\mat{M}^{(k)}}$ in the original space. The dimension of the projection space is equal to the rank of $\mat{M}^{(k)}$, which is upper bounded by $k+1$ (see Section~\ref{sec:form}). A single run of \hdsl can thus be seen as incrementally building projection spaces of increasing dimensionality.

To assess the dimensionality reduction quality of \hdsl (measured by $k$-NN classification error on the test set), we plot its performance at various iterations during the runs that generated the results of Table~\ref{tKnnErr}. We compare it to two standard dimensionality reduction techniques: random projection and PCA. We also evaluate \rpml and \pcaml, i.e., learn a similarity with OASIS on top of the RP and PCA features.\footnote{Again, we were not able to run OASIS beyond a certain dimension due to computational complexity.} Note that OASIS was tuned separately for each projection size, making the comparison a bit unfair to \hdsl. The results are shown in Figure~\ref{fErr_dim}.
As observed earlier, random projection-based approaches achieve poor performance. When the features are not too noisy (as in rcv1\_2 and rcv1\_4), PCA-based methods are better than \hdsl at compressing the space into very few dimensions, but \hdsl eventually catches up. On the other hand, PCA suffers heavily from the presence of noise (dexter and dorothea), while \hdsl is able to quickly improve upon the standard similarity in the original space. Finally, on all datasets, we observe that \hdsl converges to a stationary dimension without overfitting, unlike \pcaml which exhibits signs of overfitting on dexter and rcv1\_4 especially.

%% file: conclu.tex
\section{Concluding Remarks}
\label{sConclude}

In this work, we proposed an efficient approach to learn similarity functions from high-dimensional sparse data. This is achieved by forming the similarity as a combination of simple sparse basis elements that operate on only two features and the use of an (approximate) Frank-Wolfe algorithm. Our algorithm is completed by a novel generalization analysis which validates the design choices and highlights the robustness of our approach to high dimensions. Experiments on synthetic and real datasets confirmed the good practical behavior of our method for classification and dimensionality reduction.
The learned similarity may be applied to other algorithms that rely on a similarity function (clustering, ranking), or as a way to preprocess the data before applying another learning algorithm. We also note that \citet{Amand2017a} have recently extended our \hdsl algorithm to learn local metrics for different regions of the space in addition to the global metric.

We leave several fundamental questions for future work.
In particular, our framework could be extended to optimize a loss function related to a linear classification objective. We could then attempt to adapt our analysis to obtain generalization bounds directly for the classification error. Such bounds exist in the literature \citep[see][]{Bellet2012a,Guo2014a} but exhibit a classic dependence on the data dimension that could be avoided with our approach.
Another interesting, though challenging direction is to formally study the conditions under which a sparse ground truth similarity can be accurately recovered from similarity judgments. Inspiration could be drawn from the related problem of sparse recovery in the compressed sensing literature \citep{FoucartR13}.

%% file: acks.tex

\paragraph{Acknowledgments}
This work was partially supported by a grant from CPER Nord-Pas de Calais/FEDER DATA Advanced data science and technologies 2015-2020.
It was also  supported in  part by  the  Intelligence
Advanced Research Projects Activity (IARPA) via Department of Defense U.S. Army Research Laboratory
(DoD / ARL) contract number W911NF-12-C-0012, a
NSF  IIS-1065243,  an  Alfred.   P.  Sloan  Research  Fellowship,  DARPA  award  D11AP00278,  and  an  ARO
YIP Award (W911NF-12-1-0241).  The U.S. Government is authorized to reproduce and distribute reprints
for Governmental purposes notwithstanding any copyright annotation thereon.  The views and conclusions
contained herein are those of the authors and should
not be interpreted as necessarily representing the official policies or endorsements, either expressed or implied, of IARPA, DoD/ARL, or the U.S. Government.

%% file: appendix.tex

\section{Technical Lemmas}
\label{app:tech}

The following classic result, known as the first Hoeffding's decomposition, allows to represent a $U$-statistic as a sum of i.i.d. blocks.

\begin{lemma}[\citealp{Hoeffding1948a}]
\label{lem:Ustats}
Let $q$ : $\mathcal{Z}\times\mathcal{Z}\times\mathcal{Z}\rightarrow \mathbb{R}$ be a real-valued function. Given the i.i.d. random variables $\vct{z}_1,\vct{z}_2,...,\vct{z}_n\in \mathcal{Z}$, we have
\begin{align*}
U_n(q) &= \frac{1}{n(n-1)(n-2)} \sum_{i\neq j\neq k} q(\vct{z}_i,\vct{z}_j,\vct{z}_k) \\
&= \frac{1}{n!}\sum_\pi \frac{1}{\lfloor n/3 \rfloor} \sum_{i=1}^{\lfloor n/3 \rfloor} q(\vct{z}_{\pi(i)},\vct{z}_{\pi(i+{\lfloor n/3 \rfloor)}},\vct{z}_{\pi(i+2\times {\lfloor n/3 \rfloor)}}).
\end{align*}
\end{lemma}
\begin{proof}
Observe that $\forall i\neq j\neq k$, $q(\vct{z}_i,\vct{z}_j,\vct{z}_k)$ appears once on the left hand side and $U_n(q)$ has $\frac{1}{n(n-1)(n-2)}$ of its value, while on the right hand side it appears $(n-3)! \times \lfloor n/3 \rfloor $ times, because for each of the $\lfloor n/3 \rfloor$ positions there are $(n-3)!$ possible permutations. Thus the right hand side also has $\frac{1}{n(n-1)(n-2)}$ of its function value. We thus have the equality.
\end{proof}

The next technical lemma is based on the above representation.

\begin{lemma}
\label{lem:UInequality}
Let $Q$ be a set of functions from $\mathcal{Z}^3$ to $\mathbb{R}$. If $\vct{z}_1,\vct{z}_2,...,\vct{z}_n\in \mathcal{Z}$ are i.i.d., then we have
\if\arxiv1
$$\mathbb{E}[\sup_{q\in Q}\frac{1}{n(n-1)(n-2)} \sum_{i\neq j\neq k} q(\vct{z}_i,\vct{z}_j,\vct{z}_k)] \leq \mathbb{E}[\sup_{q\in Q} \frac{1}{\lfloor n/3 \rfloor} \sum_{i=1}^{\lfloor n/3 \rfloor} q(\vct{z}_{i},\vct{z}_{i+{\lfloor n/3 \rfloor}},\vct{z}_{i+2\times {\lfloor n/3 \rfloor}})].$$
\else
\begin{multline*}
\mathbb{E}[\sup_{q\in Q}\frac{1}{n(n-1)(n-2)} \sum_{i\neq j\neq k} q(\vct{z}_i,\vct{z}_j,\vct{z}_k)]\\
\leq \mathbb{E}[\sup_{q\in Q} \frac{1}{\lfloor n/3 \rfloor} \sum_{i=1}^{\lfloor n/3 \rfloor} q(\vct{z}_{i},\vct{z}_{i+{\lfloor n/3 \rfloor}},\vct{z}_{i+2\times {\lfloor n/3 \rfloor}})].
\end{multline*}
\fi
\end{lemma}
\begin{proof}
From Lemma \ref{lem:Ustats}, we observe that
\begin{equation*}
\begin{aligned}
&\mathbb{E}[\sup_{q\in Q}\frac{1}{n(n-1)(n-2)} \sum_{\vct{z}\neq \vct{z}' \neq \vct{z}''} q(\vct{z},\vct{z}',\vct{z}'')]\\
&= \mathbb{E}[\sup_{q\in Q}\frac{1}{n!}\sum_\pi\frac{1}{\lfloor n/3 \rfloor} \sum_{i=1}^{\lfloor n/3 \rfloor} q(\vct{z}_{\pi(i)},\vct{z}_{\pi(i+{\lfloor n/3 \rfloor)}},\vct{z}_{\pi(i+2\times {\lfloor n/3 \rfloor)}})]\nonumber\\
&\leq \frac{1}{n!}\mathbb{E}[\sum_\pi\sup_{q\in Q}\frac{1}{\lfloor n/3 \rfloor} \sum_{i=1}^{\lfloor n/3 \rfloor} q(\vct{z}_{\pi(i)},\vct{z}_{\pi(i+{\lfloor n/3 \rfloor)}},\vct{z}_{\pi(i+2\times {\lfloor n/3 \rfloor)}})]\\
&=\frac{1}{n!}\sum_\pi\mathbb{E}[\sup_{q\in Q}\frac{1}{\lfloor n/3 \rfloor} \sum_{i=1}^{\lfloor n/3 \rfloor} q(\vct{z}_{\pi(i)},\vct{z}_{\pi(i+{\lfloor n/3 \rfloor)}},\vct{z}_{\pi(i+2\times {\lfloor n/3 \rfloor)}})]\\
&=\mathbb{E}[\sup_{q\in Q}\frac{1}{\lfloor n/3 \rfloor} \sum_{i=1}^{\lfloor n/3 \rfloor} q(\vct{z}_{i},\vct{z}_{i+{\lfloor n/3 \rfloor}},\vct{z}_{i+2\times {\lfloor n/3 \rfloor}})],
\end{aligned}
\end{equation*}
which proves the result.
\end{proof}

Finally, we recall McDiarmid's inequality.

\begin{lemma}[\citealp{mcdiarmid1989method}]
\label{lem:mcdiarmid}
Let $\mathcal{Z}$ be some set and let $f:\mathcal{Z}^n\rightarrow\mathbb{R}$ be a function of $n$ variables such that for some $c>0$, for all $i\in\{1,\dots,n\}$ and for all $z_1,\dots,z_n,z_i'\in \mathcal{Z}$, we have
$$|f(z_1,\dots, z_{i-1}, z_i, z_{i+1}, \dots, ,z_n) - f(z_1,\dots,z_{i-1},z'_i,z_{i+1},\dots,z_n)|\leq c.$$
Let $Z_1,\dots,Z_n$ be $n$ independent random variables taking values in $\mathcal{Z}$. Then, with probability at least $1-\delta$, we have
$$|f(Z_1,\dots,Z_n) - \mathbb{E}[f(Z_1,\dots,Z_n)]|\leq c\sqrt{\frac{n\log(2/\delta)}{2}}.$$
\end{lemma}

\section{Proof of Lemma~\ref{lem:rademacher}}
\label{app:rad}

\begin{proof}
Given a training sample $S=\{\vct{z}_i=(\vct{x}_i,y_i) : i\in 1,\dots,n\}\sim \mu^n$, we denote the set of admissible triplets involved in the Rademacher complexity by
$$A_S = \left\{i: y_i=y_{i+{\lfloor n/3 \rfloor}}\neq y_{i+2\times {\lfloor n/3 \rfloor}}, i = 1,\dots, \lfloor n/3 \rfloor \right\},$$
and let $m=|A_S|\leq \lfloor n/3 \rfloor$. We have:
\begin{align}
R_n(\mathcal{F}^{(k)}) & = \mathbb{E}_{\vct{\sigma},S\sim\mu^n}\sup_{\mat{M}\in\mathcal{D}_\lambda^{(k)}}\frac{1}{\lfloor n/3 \rfloor}\sum_{i\in A_S} \sigma_i  \ell(\innerp{\vct{x}_{i}(\vct{x}_{i+{\lfloor n/3 \rfloor}}-\vct{x}_{i+2\times {\lfloor n/3 \rfloor}})\T,\mat{M}})\nonumber\\
&\leq \mathbb{E}_{\vct{\sigma},S\sim\mu^n}\sup_{\mat{M}\in\mathcal{D}_\lambda^{(k)}}\frac{1}{\lfloor n/3 \rfloor}\sum_{i\in A_S} \sigma_i  \innerp{\vct{x}_{i}(\vct{x}_{i+{\lfloor n/3 \rfloor}}-\vct{x}_{i+2\times {\lfloor n/3 \rfloor}})\T,\mat{M}}\label{eq:contract}\\
&= \frac{m}{\lfloor n/3 \rfloor}\mathbb{E}_{\vct{\sigma},S\sim\mu^n}\frac{1}{m}\sup_{\mat{M}\in\mathcal{D}_\lambda^{(k)}}\sum_{i\in A_S} \sigma_i  \innerp{\vct{x}_{i}(\vct{x}_{i+{\lfloor n/3 \rfloor}}-\vct{x}_{i+2\times {\lfloor n/3 \rfloor}})\T,\mat{M}}\nonumber\\
&\leq \frac{m}{\lfloor n/3 \rfloor}\max_{\vct{u}\in U}\|\vct{u}-\bar{\vct{u}}\|_2 \frac{\sqrt{2\log k}}{m}\label{eq:finiteset}\\
&= \frac{1}{\lfloor n/3 \rfloor}\max_{\vct{u}\in U}\|\vct{u}-\bar{\vct{u}}\|_2 \sqrt{2\log k}\nonumber\\
&\leq \frac{1}{\lfloor n/3 \rfloor}8\lambda B_{\mathcal{X}}\sqrt{m}\sqrt{2\log k}\label{eq:boundeddata}\\
&\leq 8\lambda B_{\mathcal{X}}\sqrt{\frac{2\log k}{\lfloor n/3 \rfloor}},\nonumber
\end{align}
where the set $U = \{\vct{u}_\tau\in \mathbb{R}^m: \tau=1,\dots,k, (\vct{u}_\tau)_i = \innerp{\vct{x}_{\gamma(i)}(\vct{x}_{{\gamma(i)}+{\lfloor n/3 \rfloor}}-\vct{x}_{{\gamma(i)}+2\times {\lfloor n/3 \rfloor}})^T,\mat{B}_\tau},  \gamma: \{1,\dots,m\} \rightarrow A_S\text{ is bijective}, \mat{B}_\tau \in \mathcal{B}_\lambda\}$, and $\bar{\vct{u}} = \frac{1}{k}\sum_{\tau=1}^k\vct{u}_\tau$. The inequality \eqref{eq:contract} follows from the contraction property \citep[see][Lemma 26.9]{Shalev-Shwartz2014a}. The inequality \eqref{eq:finiteset} follows from the fact $\mat{M}$ is a convex combination of set of $k$ bases combined with the properties in \citet[][Lemma 26.7, 26.8]{Shalev-Shwartz2014a}. Finally, inequality \eqref{eq:boundeddata} follows from the sparsity structure of the bases and the fact that $\vct{x}_i(\vct{x}_j - \vct{x}_k)^T$ has no entries with absolute value greater than $B_{\mathcal{X}}$.
\end{proof}

\section{Proof of Theorem~\ref{thm:max_deviations}}
\label{app:maxdev}

\begin{proof}
Let us consider the function
$$\Phi(S) = \sup_{\mat{M}\in\mathcal{D}_\lambda^{(k)}} [\mathcal{L}(\mat{M}) - \mathcal{L}_{S}(\mat{M}) ].$$

Let $S=\{\vct{z}_1,\dots,\vct{z}_{q-1},\vct{z}_q,\vct{z}_{q+1},\dots,\vct{z}_n\}$ and $S'=\{\vct{z}_1,\dots,\vct{z}_{q-1},\vct{z}'_q,\vct{z}_{q+1},\dots,\vct{z}_n\}$ be two samples differing by exactly one point. We have:

\begin{align*}
\Phi(S') - \Phi(S) &\leq \sup_{\mat{M}\in\mathcal{D}_\lambda^{(k)}} [\mathcal{L}_S(\mat{M}) - \mathcal{L}_{S'}(\mat{M}) ]\\
&\leq \frac{1}{n(n-1)(n-2)} \sup_{\mat{M}\in\mathcal{D}_\lambda^{(k)}} \sum_{i\neq j\neq k} |L_{\mat{M}}(\vct{z}_i,\vct{z}_j,\vct{z}_k) - L_{\mat{M}}(\vct{z}'_i,\vct{z}'_j,\vct{z}'_k)| \\
&\leq \frac{1}{n(n-1)(n-2)}6(n-1)(n-2)B_{\mathcal{X}}B_{\mathcal{D}_\lambda^{(k)}} = \frac{6}{n}B_{\mathcal{X}}B_{\mathcal{D}_\lambda^{(k)}}.
\end{align*}
The first inequality comes from the fact that the difference of suprema does not exceed the supremum of the difference. The last inequality makes use of \eqref{eq:holder}. Similarly, we can obtain $\Phi(S) - \Phi(S') \leq 6B_{\mathcal{X}}B_{\mathcal{D}_\lambda^{(k)}}/n$, thus we have $|\Phi(S) - \Phi(S')| \leq 6B_{\mathcal{X}}B_{\mathcal{D}_\lambda^{(k)}}/n$. We can therefore apply McDiarmid's inequality (see Lemma~\ref{lem:mcdiarmid} in \ref{app:tech}) to $\Phi(S)$: for any $\delta>0$, with probability at least $1-\delta$ we have:
\begin{equation}
\label{eq:eMcDiarmid}
\sup_{\mat{M}\in\mathcal{D}_\lambda^{(k)}}[\mathcal{L}(\mat{M}) -\mathcal{L}_S(\mat{M})] \leq \mathbb{E}_S\sup_{\mat{M}\in\mathcal{D}_\lambda^{(k)}}[\mathcal{L}(\mat{M}) -\mathcal{L}_S(\mat{M})] + 3B_{\mathcal{X}}B_{\mathcal{D}_\lambda^{(k)}}\sqrt{\frac{2\ln{(2/\delta)}}{n}}.
\end{equation}

We thus need to bound $\mathbb{E}_S\sup_{\mat{M}\in\mathcal{D}_\lambda^{(k)}}[\mathcal{L}(\mat{M}) -\mathcal{L}_S(\mat{M})]$.
Applying Lemma \ref{lem:UInequality} (see \ref{app:tech}) with $q_{\mat{M}}(\vct{z},\vct{z}',\vct{z}'') = \mathcal{L}(\mat{M}) - L_{\mat{M}}(\vct{z},\vct{z}',\vct{z}'')$ gives
\begin{equation*}
\mathbb{E}_S\sup_{\mat{M}\in\mathcal{D}_\lambda^{(k)}}[\mathcal{L}(\mat{M}) -\mathcal{L}_S(\mat{M})] \leq \mathbb{E}_S\sup_{\mat{M}\in\mathcal{D}_\lambda^{(k)}}[\mathcal{L}(\mat{M}) -\bar{\mathcal{L}}_S(\mat{M})],
\end{equation*}
where $\bar{\mathcal{L}}_S(\mat{M}) = \frac{1}{\lfloor n/3 \rfloor}\sum_{i=1}^{\lfloor n/3 \rfloor} L_{\mat{M}}(\vct{z}_{i},\vct{z}_{i+{\lfloor n/3 \rfloor}},\vct{z}_{i+2\times {\lfloor n/3 \rfloor}})$. Let $\bar{S} = \{\vct{\bar{z}}_1,...,\vct{\bar{z}}_n\}$ be an i.i.d. sample independent of $S$. Then
\begin{eqnarray*}
\mathbb{E}_S\sup_{\mat{M}\in\mathcal{D}_\lambda^{(k)}}[\mathcal{L}(\mat{M}) -\bar{\mathcal{L}}_S(\mat{M})] &=& \mathbb{E}_S\sup_{\mat{M}\in\mathcal{D}_\lambda^{(k)}}[\mathbb{E}_{\bar{S}}\bar{\mathcal{L}}_{\bar{S}}(\mat{M}) -\bar{\mathcal{L}}_S(\mat{M})] \\
&\leq &\mathbb{E}_{S,\bar{S}}\sup_{\mat{M}\in\mathcal{D}_\lambda^{(k)}}[\bar{\mathcal{L}}_{\bar{S}}(\mat{M}) -\bar{\mathcal{L}}_S(\mat{M})].
\end{eqnarray*}

Let $\sigma_1,\dots,\sigma_{\lfloor \frac{n}{3} \rfloor}\in\{-1,1\}$ be a collection of i.i.d. Rademacher variables. By standard symmetrization techniques, we have that
\if\arxiv1
\begin{equation*}
\begin{aligned}
&\mathbb{E}_{S,\bar{S}}\sup_{\mat{M}\in\mathcal{D}_\lambda^{(k)}}[\bar{\mathcal{L}}_{\bar{S}}(\mat{M}) -\bar{\mathcal{L}}_S(\mat{M})]\\
&= \mathbb{E}_{\vct{\sigma}, S,\bar{S}}\frac{1}{\lfloor n/3 \rfloor}\sup_{\mat{M}\in\mathcal{D}_\lambda^{(k)}}\sum_{i=1}^{\lfloor n/3 \rfloor} \sigma_i \big[ L_{\mat{M}}(\bar{\vct{z}}_{i},\bar{\vct{z}}_{i+{\lfloor n/3 \rfloor}},\bar{\vct{z}}_{i+2\times {\lfloor n/3 \rfloor}}) - L_{\mat{M}}(\vct{z}_{i},\vct{z}_{i+{\lfloor n/3 \rfloor}},\vct{z}_{i+2\times {\lfloor n/3 \rfloor}})\big]\\
&\leq \frac{1}{\lfloor n/3 \rfloor} [ \mathbb{E}_{\vct{\sigma}, \bar{S}} \sup_{\mat{M}\in\mathcal{D}_\lambda^{(k)}} \sum_{i=1}^{\lfloor n/3 \rfloor} \sigma_i L_{\mat{M}}(\bar{\vct{z}}_{i},\bar{\vct{z}}_{i+{\lfloor n/3 \rfloor}},\bar{\vct{z}}_{i+2\times {\lfloor n/3 \rfloor}}) + \mathbb{E}_{\vct{\sigma}, S} \sup_{\mat{M}\in\mathcal{D}_\lambda^{(k)}} \sum_{i=1}^{\lfloor n/3 \rfloor} \sigma_i L_{\mat{M}}(\vct{z}_{i},\vct{z}_{i+{\lfloor n/3 \rfloor}},\vct{z}_{i+2\times {\lfloor n/3 \rfloor}}) ]\\
&= 2\mathbb{E}_{\vct{\sigma}, S}\frac{1}{\lfloor n/3 \rfloor}\sup_{\mat{M}\in\mathcal{D}_\lambda^{(k)}}\sum_{i=1}^{\lfloor n/3 \rfloor} \sigma_i  L_{\mat{M}}({\vct{z}}_{i},{\vct{z}}_{i+{\lfloor n/3 \rfloor}},{\vct{z}}_{i+2\times {\lfloor n/3 \rfloor}}) = 2R_n(\mathcal{F}^{(k)}).
\end{aligned}
\end{equation*}
\else
\begin{equation*}
\begin{aligned}
&\mathbb{E}_{S,\bar{S}}\sup_{\mat{M}\in\mathcal{D}_\lambda^{(k)}}[\bar{\mathcal{L}}_{\bar{S}}(\mat{M}) -\bar{\mathcal{L}}_S(\mat{M})]\\
&= \mathbb{E}_{\vct{\sigma}, S,\bar{S}}\frac{1}{\lfloor n/3 \rfloor}\sup_{\mat{M}\in\mathcal{D}_\lambda^{(k)}}\sum_{i=1}^{\lfloor n/3 \rfloor} \sigma_i \big[ L_{\mat{M}}(\bar{\vct{z}}_{i},\bar{\vct{z}}_{i+{\lfloor n/3 \rfloor}},\bar{\vct{z}}_{i+2\times {\lfloor n/3 \rfloor}})\\
&~~~- L_{\mat{M}}(\vct{z}_{i},\vct{z}_{i+{\lfloor n/3 \rfloor}},\vct{z}_{i+2\times {\lfloor n/3 \rfloor}})\big]\\
&\leq \frac{1}{\lfloor n/3 \rfloor} [ \mathbb{E}_{\vct{\sigma}, \bar{S}} \sup_{\mat{M}\in\mathcal{D}_\lambda^{(k)}} \sum_{i=1}^{\lfloor n/3 \rfloor} \sigma_i L_{\mat{M}}(\bar{\vct{z}}_{i},\bar{\vct{z}}_{i+{\lfloor n/3 \rfloor}},\bar{\vct{z}}_{i+2\times {\lfloor n/3 \rfloor}})\\
&~~~+ \mathbb{E}_{\vct{\sigma}, S} \sup_{\mat{M}\in\mathcal{D}_\lambda^{(k)}} \sum_{i=1}^{\lfloor n/3 \rfloor} \sigma_i L_{\mat{M}}(\vct{z}_{i},\vct{z}_{i+{\lfloor n/3 \rfloor}},\vct{z}_{i+2\times {\lfloor n/3 \rfloor}}) ]\\
&= 2\mathbb{E}_{\vct{\sigma}, S}\frac{1}{\lfloor n/3 \rfloor}\sup_{\mat{M}\in\mathcal{D}_\lambda^{(k)}}\sum_{i=1}^{\lfloor n/3 \rfloor} \sigma_i  L_{\mat{M}}({\vct{z}}_{i},{\vct{z}}_{i+{\lfloor n/3 \rfloor}},{\vct{z}}_{i+2\times {\lfloor n/3 \rfloor}}) = 2R_n(\mathcal{F}^{(k)}).
\end{aligned}
\end{equation*}
\fi

We have thus shown:
\begin{equation}
\label{eq:derivRad}
\mathbb{E}_S\sup_{\mat{M}\in\mathcal{D}_\lambda^{(k)}}[\mathcal{L}(\mat{M}) -\mathcal{L}_S(\mat{M})] \leq 2R_n(\mathcal{F}^{(k)}).
\end{equation}
Plugging \eqref{eq:derivRad} into \eqref{eq:eMcDiarmid} and using Lemma~\ref{lem:rademacher}, we get the desired result.
\end{proof}

\section{Proof of Corollary~\ref{cor:excess_risk}}
\label{app:excess}

\begin{proof}
The excess risk of $\mat{M}^{(k)}$ with respect to $\mat{M}^*$ can be decomposed as follows:
\begin{multline}
\mathcal{L}(\mat{M}^{(k)}) - \mathcal{L}(\mat{M}^*) = \mathcal{L}(\mat{M}^{(k)}) - \mathcal{L}_S(\mat{M}^{(k)}) + \mathcal{L}_S(\mat{M}^{(k)}) - \mathcal{L}_S(\mat{M}_S)\\
+ \mathcal{L}_S(\mat{M}_S) - \mathcal{L}_S(\mat{M}^*) + \mathcal{L}_S(\mat{M}^*) - \mathcal{L}(\mat{M}^*)\\
\leq \underbrace{\mathcal{L}(\mat{M}^{(k)}) - \mathcal{L}_S(\mat{M}^{(k)})}_{\text{generalization error}} + \underbrace{\mathcal{L}_S(\mat{M}^{(k)}) - \mathcal{L}_S(\mat{M}_S)}_{\text{optimization error}} + \mathcal{L}_S(\mat{M}^*) - \mathcal{L}(\mat{M}^*),\label{eq:decompo}
\end{multline}
where $\mat{M}_S \in \argmin_{\mat{M}\in\mathcal{D}_\lambda} \mathcal{L}_S(\mat{M})$ is an empirical risk minimizer.

The generalization error term in \eqref{eq:decompo} can be bounded using Theorem~\ref{thm:max_deviations} (recalling that $\mat{M}^{(k)}\in\mathcal{D}_\lambda^{(k)}$ by construction), while the optimization error term is bounded by the convergence rate of our Frank-Wolfe algorithm (Proposition~\ref{prop:converge}). In the last term, $\mat{M}^*$ does not depend on $S$, hence we can use Hoeffding's inequality together with \eqref{eq:holder} and \eqref{eq:l1_bound} to obtain that for any $\delta>0$, with probability at least $1-\delta/2$:
$$\mathcal{L}_S(\mat{M}^*) - \mathcal{L}(\mat{M}^*) \leq B_{\mathcal{X}}B_{\mathcal{D}_\lambda^{(k)}}\sqrt{\frac{\log(4/\delta)}{2n}}.$$
We get the corollary by combining the above results using the union bound.
\end{proof}